\documentclass[11pt]{article}
\usepackage{amsmath,amsfonts,amssymb,amsthm}

\usepackage{fullpage}
\usepackage{hyperref}
\usepackage{bm}
\usepackage{comment}
\usepackage{booktabs}
\usepackage{multirow}
\usepackage{algorithm}
\usepackage{algorithmic}

\newcommand{\info}{\imath}         

\usepackage{amsthm,mathtools}
\usepackage{bbm} 
\newtheorem{theorem}{Theorem}
\newtheorem{lemma}{Lemma}
\theoremstyle{definition}
\newtheorem{remark}{Remark}
\newtheorem{proposition}{Proposition}
\newtheorem{definition}{Definition}
\newtheorem{assumption}{Assumption}

\newcommand{\indic}{\mathbf{1}}
\newcommand{\R}{\mathbb{R}}
\DeclareMathOperator{\KL}{\mathrm{KL}}

\newcommand{\E}{\mathbb{E}}

\newcommand{\Cov}{\operatorname{Cov}}

\usepackage{xcolor} 
\usepackage{amsthm}
\numberwithin{equation}{section}

\newcommand\blue {\textcolor{blue}}

\usepackage{biblatex} 
\title{Entropy-Based Dimension-Free Convergence and Loss-Adaptive Schedules for Diffusion Models}

\author{
  Ahmad Aghapour \thanks{Department of Mathematics, University of Michigan, Ann Arbor, 48109; email: aghapour@umich.edu. }
  \and
  Erhan Bayraktar \thanks{Department of Mathematics, University of Michigan, Ann Arbor, 48109; email: erhan@umich.edu. }
  \and
  Ziqing Zhang \thanks{Department of Mathematics, University of Michigan, Ann Arbor, 48109; email: ziqingzh@umich.edu. }
}
\date{}

\begin{document}
\maketitle

\begin{abstract}
Diffusion generative models synthesize samples by discretizing reverse-time dynamics driven by a learned score (or denoiser). Existing convergence analyses of diffusion models typically scale at least linearly with the ambient dimension, and sharper rates often depend on intrinsic-dimension assumptions or other geometric restrictions on the target distribution.
We develop an alternative, information-theoretic approach to dimension-free convergence that avoids any geometric assumptions. Under mild assumptions on the target distribution, we bound KL divergence between the target and generated distributions by $O(H^2/K)$ (up to endpoint factors), where $H$ is the Shannon entropy and $K$ is the number of sampling steps. 
Moreover, using a reformulation of the KL divergence, we propose a Loss-Adaptive Schedule (LAS) for efficient discretization of reverse SDE which is lightweight and relies only on the training loss, requiring no post-training heavy computation. Empirically, LAS improves sampling quality over common heuristic schedules.
\end{abstract}

\section{Introduction}
Diffusion-based generative models have emerged as one of the most powerful classes of deep generative models \cite{sohl2015deep,song2019generative,song2021scorebased}, achieving state-of-the-art performance in image \cite{dhariwal2021diffusion,ramesh2022hierarchical,rombach2022high}, audio \cite{kong2021diffwave,liu2023audioldm} and video \cite{ho2022video} synthesis, as well as molecule generation and protein design \cite{hoogeboom2022equivariant,corso2022diffdock,watson2023novo}. These methods construct a forward noising process that gradually corrupt data to a prior distribution (usually Gaussian), together with a reverse-time denoising process --- either a reverse stochastic differential equation (SDE) \cite{ho2020denoising} or a deterministic probability-flow ordinary differential equation (ODE) \cite{song2021denoising} --- that transports the prior distribution back to the data distribution. In practice, training a diffusion model amounts to learning the score function of the forward noising process at each noise level; to generate samples, the learned score is used to simulate the reverse-time diffusion process, which is implemented numerically as a finite sequence of denoising steps. Therefore, the main sources of sampling error of a diffusion-based model can be decomposed into statistical error in learning the score function and numerical error from time discretization of the reverse dynamics. 

Recently, there has been substantial progress in theoretical understanding of time discretization error of diffusion-model samplers; yet most available guarantees for discretization error still exhibit at least a linear dependence on the ambient dimension $d$ \cite{benton2024nearly,li2025odt}. In contrast, empirically, diffusion models produce high-quality samples with tens to hundreds of steps in very high ambient dimensions. This gap between theory and practice suggests that linear-in-$d$ bounds can be overly conservative. 

There are results that improve upon linear-in-$d$ scaling, but they typically come at a cost: they either impose restrictive structural or geometric assumptions on the target distribution, or they establish convergence only in a weaker functional metric than the commonly used total variation (TV) or KL-divergence. For instance, \cite{bruno2025on} derives an error bound of order $\sqrt{d}$ under a log-concavity assumption, while \cite{li2025dimension} proves a dimension-free bound assuming the target is well-approximated by a Gaussian mixture. In settings where the data distribution concentrates on a low-dimensional subspace or manifold, \cite{li2025odt,liang2025low,pmlr-v291-potaptchik25a} obtain discretization error bounds that scale linearly with an intrinsic dimension rather than the ambient dimension. Finally, \cite{de2025dimension} establishes a dimension-free discretization bound, but in a weaker functional metric based on smooth test functionals.

In this work, we take a different route: we derive a dimension-free discretization bound expressed in terms of Shannon entropy, and our bound only relies on a mild assumption on the information content, but not any geometric regularity assumptions.

Our main contributions can be summarized as follows:
\begin{itemize}
    \item We decompose the sampling error in KL-divergence into a \emph{score estimation error} term and a \emph{time discretization error} term, and express the discretization error as a minimum mean square error (MMSE) functional.
    \item Under mild assumptions on the information content of the data distribution, we establish a \emph{dimension-free} discretization error bound of order $O(H^2 / K)$, where $H$ is the Shannon entropy of the data distribution and $K$ is the number of discretization steps.
\end{itemize}

In addition to these theoretical results, we develop a practical, empirically motivated component that leverages an intermediate insight from our analysis:
\begin{itemize}
    \item We propose a Loss-Adaptive Schedule (LAS) for discretizing the reverse SDE that empirically outperforms standard heuristics. The schedule is computationally efficient, relying only on the training loss, which is available or cheaply estimable at the end of training.
\end{itemize}

\section{Related Work}
Theoretical analysis of diffusion-based generative models typically decomposes sampling error into score estimation error and discretization error from numerically integrating the reverse SDE or its associated probability-flow ODE. 
For DDPM-style stochastic samplers \cite{ho2020denoising} in particular, \cite{chen2023sampling} proved one of the first general non-asymptotic guarantees for diffusion models that scales polynomially in problem parameters, without strong assumptions on the data distribution, such as log-concavity and functional inequality. They showed that with $L$-Lipschitzness of score function and a score estimator with $L_2$-error at most $\tilde O(\varepsilon)$, a discrete-time reverse diffusion sampler outputs a measure which is $\varepsilon$-close in TV distance to the true data distribution in $\tilde O(L^2 d/\varepsilon^2)$ iterations. 
\cite{chen2023improved} refined this analysis by providing a KL-divergence bound of order $\tilde O (d\log(1/\delta)/\varepsilon)$ for the variance-$\delta$ Gaussian perturbation of any data distribution and under a relaxed $1/\delta$-smoothness assumption on score functions. 

Later, a number of works sharpen the dimension dependence in these global complexity bounds.
\cite{benton2024nearly} derived the first nearly $d$-linear convergence guarantees via stochastic localization under only finite second-moment assumptions. \cite{li2025odt} proved TV complexity bound of order $\tilde O (d/\varepsilon)$ for any target distribution with finite first-order moment.
Beyond ambient-dimension-dependent theory, \cite{li2025odt,liang2025low,pmlr-v291-potaptchik25a} studied adaptivity to low intrinsic dimension and proved convergence rate of $\tilde O(k/\varepsilon)$ in terms of TV and KL-divergence, where $k$ is the intrinsic dimension of the target data distribution. 

There are a few recent works regarding dimension independent bounds. \cite{li2025dimension} obtained TV bound of order $\tilde O (1/\varepsilon)$ for data distribution well-approximated by Gaussian mixture models. 
\cite{de2025dimension} derived a dimension-free bound in a weaker functional metric, defined by smooth test functionals with bounded first and second derivatives. \cite{gatmiry2026high} proposed a new collocation-based sampler and proved an iteration complexity logarithmic in $1/\varepsilon$ and no explicit dependence on the ambient dimension; however the dimension enters indirectly through the effective radius of the support of the target distribution.

To date, theoretical bounds on discretization error for diffusion samplers typically scale with the ambient or intrinsic dimension, and dimension-free guarantees are usually obtained only under restrictive structural assumptions on the target distribution or in a weaker functional metric. In this work, we derive a dimension-free discretization bound using an information-theoretic route that imposes no geometric regularity conditions. Concretely, our analysis first yields a dimension-free control in terms of the order-$1/2$ R\'enyi entropy $H_{1/2}$; this intermediate bound requires \emph{no additional assumptions} beyond the finiteness of $H_{1/2}$. We then introduce a mild concentration condition on the information content solely to relate $H_{1/2}$ to the Shannon entropy $H$, resulting in a final dimension-free discretization bound expressed in terms of $H$.
\section{Problem Setup}
Given training samples from a target data distribution $p$, a diffusion model seeks to generate new samples from $p$. It consists of a forward noising process, which progressively corrupts data samples by adding Gaussian noise, and a learned reverse-time denoising process, which synthesizes new data by reversing this corruption and transforming noisy samples back toward the data distribution.

\paragraph{Forward process.}
For simplicity, we consider the standard forward diffusion given by Brownian motion started from a data distribution $p$ on $\mathbb{R}^d$:
\begin{equation*}
dX_t = dW_t,\qquad X_0\sim p,\quad t\in[0,T].
\end{equation*}
where $(W_t)_{t\in[0,T]}$ is a Brownian motion on $\mathbb{R}^d$, independent of $X_0$. Throughout, we write $Z:=X_0$ for the underlying data sample.

\paragraph{Reverse process.}

Let $p_t$ denote the law of $X_t$, and define the (Bayes-optimal) denoiser
\begin{equation*}
m_t(x) := \mathbb{E}[Z\mid X_t=x],\qquad t>0.
\end{equation*}
By Tweedie’s formula for Gaussian convolution,
\begin{equation*}
\nabla\log p_t(x) = \frac{m_t(x)-x}{t}.
\end{equation*}

Define the reverse-time process $(Y_s)_{s\in[0,T]}$ by $Y_s := X_{T-s}$ for all $s\in[0,T]$. 
Then, under mild conditions satisfied by the processes considered in this work \cite{haussmann1986time}, the reverse process admits an SDE description
\begin{equation*}
dY_s = \beta_s(Y_s)\,ds + dB_s,\qquad
Y_0 = X_{T},
\end{equation*}
on $s\in[0,T)$, where 
\[\beta_s(y) = \nabla\log p_{T-s}(y) = \frac{m_{T-s}(y)-y}{T-s},\]
and $(B_s)_{s\in[0,T]}$ is a Brownian motion.

\paragraph{Approximate reverse process.}

In practice, this reverse-time process is implemented on a discretized time grid, and and the unknown denoiser $m$ is replaced by an approximation $\hat m$ obtained from a learned score function.

Fix a small $\delta>0$ and let $T_\delta:=T-\delta$.
Fix a time grid
\begin{equation*}
0 = s_0 < s_1 < \dots < s_K = T_\delta < T.
\end{equation*}
We now define an approximate reverse process $(\tilde Y_s)_{s\in[0,T_\delta]}$ by replacing the true denoiser $m_{T-s}$ with an approximation $\hat m_{T-s_{k-1}}$ whose time index is frozen on each interval, and whose value is evaluated at the previous gridpoint state. Especially, we keep all other terms in the drift of reverse process unchanged. Concretely, for $s\in(s_{k-1},s_k]$ we use the drift
\begin{equation*}
\tilde\beta^{(k)}_s(y)
:= \frac{\hat m_{T-s_{k-1}}(Y_{s_{k-1}}) - y}{T-s}\,.
\end{equation*}

On the discretization grid, we have 
\begin{equation*}
\tilde Y_{s_k}-\tilde Y_{s_{k-1}}
= \int_{s_{k-1}}^{s_k}\tilde\beta^{(k)}_s(\tilde Y_s)\,ds
+ (B_{s_k}-B_{s_{k-1}}).
\end{equation*}
for $ k=1,\dots,K.$

The approximate reverse process $(\tilde Y_s)_{s\in[0,T_\delta]}$ satisfies
\begin{equation*}
d\tilde Y_s = \tilde\beta_s(\tilde Y_s)\,ds + dB_s,\qquad
\tilde Y_0 = X_{T},
\end{equation*}
on $s\in[0,T_\delta]$, where 
\(\tilde\beta_s(y) = \tilde\beta^{(k)}_s(y),\) for $s\in(s_{k-1},s_k]$ and $k=1,\dots,K$.

Note that we couple $(Y,\tilde Y)$ using the same Brownian increments on each interval, and we assume the same initial condition
$Y_0=\tilde Y_0=X_T$.

\paragraph{Error decomposition.}
Finally, we express the sampling error in terms of the KL divergence between the sampling distribution and the target distribution.

On each interval $(s_{k-1},s_k]$ define the drift mismatch
\begin{equation*}
\delta_s := \beta_s(Y_s)-\tilde\beta^{(k)}_s(Y_s).
\end{equation*}
Let $\mathbb{P}$ and $\tilde{\mathbb{P}}$ denote the path laws of $Y$ and $\tilde Y$ on $C([0,T_\delta],\mathbb{R}^d)$.

\begin{proposition}\label{prop:KL}
Assume the square-integrability condition
\begin{equation}\label{eq:finite-energy-under-P}
\mathbb{E}_{\mathbb{P}}\!\left[\int_0^{T_\delta}\|\delta_s\|^2ds\right]<\infty.
\end{equation}
The total pathwise KL has upper bound 
\begin{align}
\mathrm{KL}({\mathbb{P}}\|\tilde{\mathbb{P}})
&\leq \frac12\,\mathbb E_{\mathbb{P}}\left[\int_0^{T_\delta}\|\delta_s\|^2\,ds\right]\notag\\
&=\frac12\sum_{k=1}^K
\mathbb E_{{\mathbb{P}}}\left[\int_{s_{k-1}}^{s_k}
\big\|\beta_s(Y_s)-\tilde\beta^{(k)}_s(Y_s)\big\|^2\,ds
\right].
\label{eq:KL-main}
\end{align}
\end{proposition}
\begin{remark}
\eqref{eq:finite-energy-under-P} is the natural condition for the right-hand side of the desired bound to be finite, and ensures that the stochastic integral $\int_0^t \delta_s^\top dB_s$ is well-defined on $[0,T_\delta]$.
\end{remark}

Equation~\eqref{eq:KL-main} controls the discrepancy between the path laws of the exact and approximate reverse processes. 
Our ultimate goal, however, is to control the discrepancy between the terminal sampling distributions at time $T_\delta$ (i.e., the laws of $Y_{T_\delta}$ and $\tilde Y_{T_\delta}$).
Let $\mathbb{P}_{T_\delta}$ and $\tilde{\mathbb{P}}_{T_\delta}$ denote the pushforwards of $\mathbb{P}$ and $\tilde{\mathbb{P}}$ under the evaluation map $\omega\mapsto \omega(T_\delta)$.
Since this map is measurable, the data processing inequality for KL divergence gives
\[
\mathrm{KL}(\mathbb{P}_{T_\delta}\,\|\,\tilde{\mathbb{P}}_{T_\delta})
\le
\mathrm{KL}(\mathbb{P}\,\|\,\tilde{\mathbb{P}}).
\]
Therefore, any upper bound on the pathwise KL in \eqref{eq:KL-main} immediately yields an upper bound on the KL divergence between the sampling distribution produced by the discretized reverse dynamics and the true reverse-time marginal at time $T_\delta$.
In particular, it suffices to bound the right-hand side of \eqref{eq:KL-main}.

We now split the drift mismatch into two contributions: (i) the error from freezing the time index and using the previous gridpoint state, and (ii) the error from replacing the true denoiser by $\hat m$.

Insert and subtract $m_{T-s_{k-1}}(Y_{s_{k-1}})$:
\begingroup
\small
\begin{align*}
\beta_s(Y_s)-\tilde\beta^{(k)}_s(Y_s)
= & \frac{m_{T-s}(Y_s)-m_{T-s_{k-1}}(Y_{s_{k-1}})}{T-s} \\
& + \frac{m_{T-s_{k-1}}(Y_{s_{k-1}})-\hat m_{T-s_{k-1}}(Y_{s_{k-1}})}{T-s}.
\end{align*}
\endgroup

Denote $t:=T-s$ and $t_{k-1}:=T-s_{k-1}$. Let $\mathcal F_{t}$ be the filtration defined by $\mathcal{F}_t = \sigma\bigl( X_u : t \leq u \leq T \bigr)$. 
For $t\le t_{k-1}$, since $\mathcal F_{{t_{k-1}}}\subset\mathcal F_t$, using the tower property and Markov property, we have
\begin{align*}
\mathbb E\!\left[m_t(X_t)-m_{t_{k-1}}(X_{t_{k-1}})\mid \mathcal F_{{t_{k-1}}}\right]
= 0.
\end{align*}
Since $(Y_s)_{s\in[0,T_\delta]} = (X_{T-s})_{s\in[0,T_\delta]}$, the cross term vanishes on decomposing $\beta_s(Y_s)-\tilde\beta^{(k)}_s(Y_s)$ and the split into discretization and approximation contributions is exact: 
\begin{equation}
\mathrm{KL}(\mathbb{P}\|\tilde{\mathbb{P}})
= \frac{1}{2 } ( \mathcal{E}_{\mathrm{disc}} + \mathcal{E}_{\mathrm{apx}}),
\label{eq:KL-split}
\end{equation}

where
\begingroup
\begin{equation}\label{eq:disc_err}
\mathcal{E}_{\mathrm{disc}}
:= \sum_{k=1}^K
\mathbb{E}_{\mathbb{P}}\left[
\int_{s_{k-1}}^{s_k}
\frac{\big\|m_{T-s}(Y_s)-m_{T-s_{k-1}}(Y_{s_{k-1}})\big\|^2}{(T-s)^2}\,ds
\right],
\end{equation}
\begin{equation*}
\mathcal{E}_{\mathrm{apx}}
:= \sum_{k=1}^K
\mathbb{E}_{\mathbb{P}}\left[
\int_{s_{k-1}}^{s_{k}}
\frac{\big\|m_{T-s_{k-1}}(Y_{s_{k-1}})-\hat m_{T-s_{k-1}}(Y_{s_
{k-1}})\big\|^2}{(T-s)^2}\,ds
\right].
\end{equation*}
\endgroup
The term $\mathcal{E}_{\mathrm{apx}}$ is driven entirely by the quality of the estimator $\hat m$ (equivalently, the learned score), and corresponds to the statistical error. The term $\mathcal{E}_{\mathrm{disc}}$ is the numerical time discretization error of the reverse dynamics; it persists even if one had access to the exact denoiser, and it is the object of our main dimension-free control.

\section{Main Results}
In this section, we state our main results. Our first goal is to rewrite the discretization error $\mathcal{E}_{\mathrm{disc}}$ in a functional form that depends only on the MMSE along the forward Gaussian channel. This representation turns the discretization analysis into a problem of controlling how the MMSE varies with the SNR. Our second goal is to obtain a \emph{dimension-free} upper bound on $\mathcal{E}_{\mathrm{disc}}$ by proving an explicit bound on the derivative of the MMSE. Throughout, we work under a mild information-theoretic condition on the target distribution, stated next.

\begin{definition}
For a probability density function or probability mass function $p$ supported on \(\mathcal C\subset\R^d\), define the information content for any $z\in \mathcal{C}$ as
\[
\info(z):=\log\frac{1}{p(z)}.
\]
Define the Shannon entropy of $p$ as
\[H:=\mathbb E_{Z\sim p}[\info(Z)]\,.\]
\end{definition}
We impose two mild moment/tail conditions on the target distribution.

\begin{assumption}\label{assumption0}
For $Z\sim p$, we have $\E\|Z\|^2<\infty$.
\end{assumption}

\begin{assumption}\label{assumption}
Suppose the target data distribution $p$ is discrete and supported on a countable set \(\mathcal C\subset\R^d\). 
Assume $H<\infty$ and the information content is sub-exponential about its mean, \textit{i.e.}
there exist constants $\nu^2>0$ and $b\in(0,2]$ such that for all
$\lambda\in\R$ with $|\lambda|\le 1/b$,
\begin{equation}\tag{SE}\label{eq:SE-again}
\mathbb E\exp\big(\lambda(\info(Z)-H)\big)
\le \exp\big(\nu^2\lambda^2\big).
\end{equation}
\end{assumption}

\begin{remark}

Assumption~\ref{assumption} models the target distribution $p$ as discrete on a countable subset of
$\mathbb{R}^d$. This matches a common and practically relevant setting in diffusion modeling: latent diffusion
models (LDMs) built on a vector-quantized (VQ) first stage. In VQ-based representations, each latent is obtained
by selecting codebook indices from a finite set and mapping them to real-valued codebook embeddings. Although the
embeddings are vectors in $\mathbb{R}^d$, the latent variable itself takes values in a finite (hence countable)
subset of $\mathbb{R}^d$.
\end{remark}
\begin{remark}
Assumption~\ref{assumption} is a tail condition on the information content (surprisal)
\(
\info(Z)=-\log p(Z)
\)
for \(Z\sim p\). A useful interpretation is \emph{codelength}: \(\info(z)\) is the ideal number of nats required to encode \(z\) under an optimal code matched to \(p\), and \(H=\E[\info(Z)]\) is the average codelength (Shannon entropy). For natural-image distributions, most samples are \emph{highly structured and compressible}, so their codelengths cluster around a typical value; images that require substantially more bits than average correspond to atypical, ``hard-to-compress'' configurations and are expected to be rare. Assumption~\ref{assumption} formalizes this typicality intuition by requiring that upward deviations of \(\info(Z)\) above \(H\) have sub-exponential tails.

A complementary lens is to view \(p\) as a local-energy (Gibbs/EBM) model:
\[
p(z)=\frac{1}{\mathcal Z}\exp\!\big(-E(z)\big),
\qquad
E(z)=\sum_{a\in\mathcal A}\phi_a(z_a),
\]
where \(z=(z_i)_{i\in V}\) denotes the image/latent indexed by sites \(V\), \(\mathcal A\) is a collection of
\emph{factor scopes} \(a\subseteq V\) (e.g.\ overlapping patches at one or multiple scales), \(z_a\) is the restriction of \(z\) to the sites in \(a\), and \(\phi_a\) penalizes locally implausible configurations on that scope. In this representation,
\[
\info(z)=-\log p(z)=E(z)+\log \mathcal Z,
\]
so concentration of surprisal is equivalent to concentration of the energy \(E(Z)\) (since \(\log\mathcal Z\) is constant).

To connect this to image structure, it is helpful to express \(\info(Z)-H\) as an accumulation of
\emph{incremental surprises} as the image is revealed progressively. Fix any ordering
\(\mathcal A=\{a_1,\dots,a_M\}\) of the factor scopes and define
\(S_j:=\cup_{i\le j}a_i\) and \(\mathcal F_j:=\sigma(Z_{S_j})\).
Let \(M_j:=\E[\info(Z)\mid \mathcal F_j]\). Then \((M_j)_{j=0}^M\) is a Doob martingale and
\[
\info(Z)-H
= M_M-M_0
=\sum_{j=1}^M D_j,
\qquad D_j:=M_j-M_{j-1},
\]
with \(\E[D_j\mid \mathcal F_{j-1}]=0\). The increment \(D_j\) can be viewed as the \emph{additional codelength or suprise}
incurred when the next neighborhood \(Z_{a_j}\) is revealed beyond what was already predictable from the
previously revealed context \(\mathcal F_{j-1}\).

For natural images, previously revealed neighborhoods typically impose strong
constraints on what comes next: edges tend to continue, textures repeat locally, and coarse-to-fine
consistency restricts fine details. This makes the ``new information'' contributed by a typical next patch
comparatively stable rather than wildly varying. Moreover, the local-energy decomposition captures a
\emph{locality of influence}: revealing \(Z_{a_j}\) can only affect the conditional expectation
\(M_j=\E[\info(Z)\mid\mathcal F_j]\) through those potentials \(\phi_a\) whose scopes overlap \(a_j\), so a single
revelation is not expected to drastically alter the model's assessment of the entire image unless long-range
interactions are extreme. Together, these considerations provide intuition for why the martingale increments
\((D_j)\) should be typically moderate, and why the total surprisal \(\info(Z)\) should concentrate around its
mean \(H\) with sub-exponential upper tails as postulated in Assumption~\ref{assumption}.
\end{remark}

\begin{remark}
Assumption~\ref{assumption} is \emph{purely information-theoretic}: it constrains only the fluctuations of the information content $\info(Z)=-\log p(Z)$ around its mean $H$, where $Z\sim p$. In particular, it imposes \emph{no} geometric or norm-based regularity on $p$---for example, it does not assume log-concavity, smoothness, manifold/subspace structure or intrinsic dimension.

\end{remark}
\begin{remark}
More generally, the sub-exponential condition in Assumption~\ref{assumption} can be replaced by the weaker requirement that the target distribution has finite R\'enyi entropy of order $1/2$, \textit{i.e.}
\[
H_{1/2}
:= \frac{1}{1-\tfrac12}\log\sum_{z\in\mathcal C}p(z)^{1/2}
= 2\log\sum_{z\in\mathcal C}\sqrt{p(z)}
<\infty\,.
\]
Under this weaker assumption, the same discretization analysis goes through with the dependence on the data distribution entering through $H_{1/2}$ , and the resulting dimension-free discretization bound scales as
\[
\mathcal{E}_{\mathrm{disc}} \;\lesssim\; \frac{H_{1/2}^2}{K}
\]

\end{remark}


\paragraph{Discretization error as an MMSE functional.}
Recall that the forward process is Brownian motion started from the data:
\(
X_t = Z + W_t,\ t\in[0,T],
\)
with $Z:=X_0\sim p$, where $(W_t)_{t\in[0,T]}$ is independent of $Z$. 
Let
\(
m_t(x) := \E[Z\mid X_t=x]
\)
denote the Bayes-optimal denoiser at noise level $t$.
We measure the denoising error through the minimum mean-squared error (MMSE) along the Gaussian channel,
parameterized by the signal-to-noise ratio (SNR) $\gamma := 1/t$:
\begin{equation*}
\mathrm{mmse}(\gamma) := \E\!\left[\|Z - m_{1/\gamma}(X_{1/\gamma})\|_2^2\right],\qquad \gamma>0.
\end{equation*}

Let $\{s_k\}_{k=0}^K$ be the reverse-time grid on $[0,T-\delta]$ and define the corresponding SNR grid
\begin{equation*}
\gamma_k := \frac{1}{T-s_k},\qquad k=0,1,\dots,K,
\end{equation*}
so that $\gamma_0=1/T$ and $\gamma_K = 1/\delta$.
Our first step is to express the reverse-time discretization error $\mathcal E_{\mathrm{disc}}$ (defined in \eqref{eq:disc_err})
as a functional of $\mathrm{mmse}(\cdot)$.

\begin{proposition}\label{thm:disc-mmse}
The discretization error $\mathcal E_{\mathrm{disc}}$ in \eqref{eq:disc_err} satisfies
\begin{equation}\label{eq:DiscErr-mmse-snr}
\mathcal E_{\mathrm{disc}}
=
\sum_{k=1}^K
\int_{\gamma_{k-1}}^{\gamma_k}
\big(\mathrm{mmse}(\gamma_{k-1})-\mathrm{mmse}(\gamma)\big)\,d\gamma.
\end{equation}
\end{proposition}

Identity \eqref{eq:DiscErr-mmse-snr} shows that $\mathcal E_{\mathrm{disc}}$ is the cumulative \emph{area gap}
between $\mathrm{mmse}(\gamma)$ and its left-endpoint values over each SNR interval. This representation reduces control of $\mathcal E_{\mathrm{disc}}$ to understanding $\mathrm{mmse}(\gamma)$.

\paragraph{From MMSE to an entropy-controlled bound.}
Since $\mathrm{mmse}(\gamma)$ is nonincreasing in $\gamma$, we can bound the area gap on each interval using the slope of $\mathrm{mmse}$:
\begingroup
\small
\begin{equation}\label{eq:disc-taylor}
\mathrm{mmse}(\gamma_{k-1})-\mathrm{mmse}(\gamma)
\le (\gamma-\gamma_{k-1})\sup_{\xi\in[\gamma_{k-1},\gamma]}(-\mathrm{mmse}'(\xi)),
\end{equation}
\endgroup
and integrating \eqref{eq:disc-taylor} over $\gamma\in[\gamma_{k-1},\gamma_k]$ yields
\begin{equation}\label{eq:disc-bound-mmseprime}
\mathcal E_{\mathrm{disc}}
\le
\sum_{k=1}^K \frac{(\Delta\gamma_k)^2}{2}
\sup_{\gamma\in[\gamma_{k-1},\gamma_k]}\big(-\mathrm{mmse}'(\gamma)\big),
\end{equation}
where $\Delta\gamma_k:=\gamma_k-\gamma_{k-1}.$

The key technical input is an entropy-based control of the MMSE derivative.

\begin{theorem}\label{thm:deriv_mmse_bound}
Under Assumption~\ref{assumption0} and~\ref{assumption}, there exists a constant $C>0$ (depending only on $(\nu,b)$) such that for all $\gamma>0$,
\begin{equation}\label{eq:deriv_mmse_bound}
|\mathrm{mmse}'(\gamma)| \le \frac{C^2 H^2}{\gamma^2}.
\end{equation}
\end{theorem}

Combining \eqref{eq:disc-bound-mmseprime} and \eqref{eq:deriv_mmse_bound} gives
\begin{equation}\label{eq:disc-bound-grid}
\mathcal E_{\mathrm{disc}}
\le
\frac{C^2H^2}{2}\sum_{k=1}^K \frac{(\Delta\gamma_k)^2}{\gamma_{k-1}^2}.
\end{equation}

\paragraph{Choosing the SNR grid: geometric spacing.}
Given \eqref{eq:disc-bound-grid}, we obtain an upper bound for $\mathcal E_{\mathrm{disc}}$ by optimizing over all the SNR grid $\{\gamma_k\}$ subject to fixed endpoints.
Let $r_k := \gamma_k/\gamma_{k-1}>0$. Then
\(
\Delta\gamma_k = \gamma_{k-1}(r_k-1)
\)
and hence
\[
\frac{(\Delta\gamma_k)^2}{\gamma_{k-1}^2} = (r_k-1)^2.
\]
Moreover the endpoint constraint becomes
\[
\prod_{k=1}^K r_k = \frac{\gamma_K}{\gamma_0} =\frac{T}{\delta}=: \Lambda.
\]
Therefore, minimizing the bound \eqref{eq:disc-bound-grid} reduces to
\[
\min\Big\{\sum_{k=1}^K (r_k-1)^2:\ r_k>0,\ \prod_{k=1}^K r_k = \Lambda\Big\}.
\]
By symmetry (and convexity of $x\mapsto (e^x-1)^2$ after the change of variables $x_k=\log r_k$),
the minimum is attained when all ratios are equal, $r_k\equiv r$, hence $r^K=\Lambda$ and $r=\Lambda^{1/K}$.
Equivalently, the optimal grid is \emph{geometric} or \emph{log-linear} in SNR:
\begin{equation}\label{eq:gamma-geometric}
\gamma_k = \gamma_0 \Lambda^{k/K},\qquad k=0,1,\dots,K.
\end{equation}
We remark that this coincides with the widely used ``log SNR'' discretization heuristic in diffusion sampling, and here it emerges as the minimizer of the upper bound.

For this choice of SNR grid,
\[
\sum_{k=1}^K \frac{(\Delta\gamma_k)^2}{\gamma_{k-1}^2}
= \sum_{k=1}^K (\Lambda^{1/K}-1)^2
= K(\Lambda^{1/K}-1)^2,
\]
and from \eqref{eq:disc-bound-grid} we conclude the dimension-free discretization bound
\begin{equation}\label{eq:disc-final-bound}
\mathcal E_{\mathrm{disc}}
\le
\frac{C^2H^2}{2}\,K\big(\Lambda^{1/K}-1\big)^2.
\end{equation}

The same geometric grid also yields a clean expression for the statistical (approximation) term
$\mathcal E_{\mathrm{apx}}$ in \eqref{eq:KL-split} when the learned model is parameterized as an
$\varepsilon$-predictor.  For $\gamma>0$, write the Gaussian channel as
\[
X_{1/\gamma}=Z+\frac{1}{\sqrt{\gamma}}\,\varepsilon,\qquad Z:=X_0\sim p,\ \ \varepsilon\sim\mathcal N(0,I_d),\ 
\]
and define the Bayes-optimal noise predictor
\[
\varepsilon_\gamma^\star(x):=\sqrt{\gamma}\bigl(x-m_{1/\gamma}(x)\bigr)=\E\bigl[\varepsilon \mid X_{1/\gamma}=x\bigr].
\]
Given any learned predictor $\hat\varepsilon_\gamma(\cdot)$, define the induced learned denoiser
\[
\hat m_{1/\gamma}(x):=x-\frac{1}{\sqrt{\gamma}}\,\hat\varepsilon_\gamma(x).
\]

\begin{proposition}\label{prop:Eapprox-logsnr}
Let $\{\gamma_k\}_{k=0}^K$ be the SNR grid \eqref{eq:gamma-geometric} .
Define the per-level $\varepsilon$-prediction MSE
\[
\epsilon_k
:=
\E\Big[\big\|\varepsilon_{\gamma_{k-1}}^\star(X_{1/\gamma_{k-1}})
-\hat\varepsilon_{\gamma_{k-1}}(X_{1/\gamma_{k-1}})\big\|_2^2\Big]
\]
for $ k=1,\dots,K.$ Then
\begin{equation}\label{eq:Eapprox-logsnr}
\mathcal E_{\mathrm{apx}}
=
\big(\Lambda^{1/K}-1\big)\cdot \sum_{k=1}^K \epsilon_k.
\end{equation}
\end{proposition}

Combining Proposition~\ref{prop:Eapprox-logsnr} with the discretization bound \eqref{eq:disc-final-bound} 
yields
\begingroup
\small
\begin{align*}
\KL(\mathbb P\|\tilde{\mathbb P})
 & =\frac12\big(\mathcal E_{\mathrm{disc}}+\mathcal E_{\mathrm{apx}}\big) \nonumber  \\ 
& \le
\frac{K}{2}\big(\Lambda^{1/K}-1\big)
\left[
\frac{C^2H^2}{2}\big(\Lambda^{1/K}-1\big)
+\frac{1}{K}\sum_{k=1}^K \epsilon_k
\right].
\end{align*}
\endgroup
In particular, if $K\ge \log \Lambda$, then $\Lambda^{1/K} = e^{(\log \Lambda)/K}\le 1+2(\log \Lambda)/K$, which implies
\begingroup
\small
\begin{align}
\mathrm{KL}(\mathbb{P}_{T_\delta}\,\|\,\tilde{\mathbb{P}}_{T_\delta})
&\le
\mathrm{KL}(\mathbb{P}\,\|\,\tilde{\mathbb{P}})\nonumber\\ 
&\leq \log \Lambda
\left[
\frac{C^2H^2}{K}\log \Lambda
+\frac{1}{K}\sum_{k=1}^K \epsilon_k
\right] \label{eq:KL-total-logsnr}.
\end{align}
\endgroup

Equation~\eqref{eq:KL-total-logsnr} shows that, under the geometric (log-SNR) grid, the sampling error admits a clean
two-term control: a \emph{dimension-free} discretization contribution governed only by the entropy $H$, and a
statistical contribution given by the average per-level prediction errors $\epsilon_k$ of the learned model.

\section{Do Not Throw Away the Training Loss!}



In \eqref{eq:KL-total-logsnr} we see a limitation of choosing the schedule by analyzing
$\mathcal E_{\mathrm{disc}}$ alone: even if geometric spacing is optimal for our upper bound on discretization, the
\emph{total} error that matters in practice is $\mathcal E_{\mathrm{disc}}+\mathcal E_{\mathrm{apx}}$, and the
approximation term depends on how the model error $\epsilon_k$ is distributed across noise levels. This suggests
that the best sampling schedule should not be universal, but instead adapt to the trained model.

In this section, we take this perspective seriously and ask: \emph{given a trained diffusion model, can we use
the information already present in its training loss to choose a better discretization schedule?}

We rewrite $\mathcal E_{\mathrm{disc}}+\mathcal E_{\mathrm{apx}}$ in terms of $x_0$-prediction risks ---quantities directly tied to standard $\varepsilon$-training losses --- and an MMSE functional. We then optimize a regularized version of the resulting expression to derive a new discretization schedule, which can be computed at negligible post-training cost. Numerical experiments show that the proposed schedule consistently outperforms the heuristics commonly used in practice. Moreover, computing the new schedule is inexpensive: it relies only on quantities that are already available at the end of training, or can be estimated at low cost. This stands in contrast to \cite{sabour2024align}, which proposes a discretization schedule through optimizing a similar KL upper bound but requires additional Monte Carlo sampling after training.

The next proposition
shows that the sum $\mathcal E_{\mathrm{disc}}+\mathcal E_{\mathrm{apx}}$ can be rewritten in terms of the model's $x_0$-prediction risk, a quantity that is already
available (or can be cheaply estimated) at the end of training.

\begin{proposition}
\label{prop:disc+approx}
For each $k=1,\dots,K$, define the model $x_0$-prediction risk at SNR
$\gamma_{k-1}$ by
\[
\mathcal L_{x_0}(\gamma_{k-1})
:=\E\!\left[\big\|Z-\hat m_{1/\gamma_{k-1}}(X_{1/\gamma_{k-1}})\big\|_2^2\right].
\]
Then $\mathcal E_{\mathrm{disc}}+\mathcal E_{\mathrm{apx}}$ would be equal to 
\begingroup
\small
\begin{equation}
\label{eq:disc+approx}
\sum_{k=1}^K(\gamma_k-\gamma_{k-1})\,\mathcal L_{x_0}(\gamma_{k-1})
\;-\;
\int_{\gamma_0}^{\gamma_K}\mathrm{mmse}(\gamma)\,d\gamma.
\end{equation}
\endgroup
\end{proposition}

A key feature of \eqref{eq:disc+approx} is that the integral term
$\int_{\gamma_0}^{\gamma_K}\mathrm{mmse}(\gamma)\,d\gamma$ depends only on the endpoints and is therefore
independent of the discretization schedule. Consequently, for fixed $K$ and fixed endpoints,
minimizing $\mathcal E_{\mathrm{disc}}+\mathcal E_{\mathrm{apx}}$ (and hence tightening the KL control) is
equivalent to minimizing the weighted sum
\[
\min_{\gamma_0<\gamma_1<\cdots<\gamma_K}\;
\sum_{k=1}^K(\gamma_k-\gamma_{k-1})\,\mathcal L_{x_0}(\gamma_{k-1}).
\]

Note the quantity controlled by
$\mathcal E_{\mathrm{disc}}+\mathcal E_{\mathrm{apx}}$ is a pathwise KL divergence, and we only access the
terminal discrepancy through the data-processing inequality; consequently, directly optimizing the bound can be
inefficient. Empirically, the looseness is most pronounced at large SNR (near $\gamma_K=1/\delta$), where the
pathwise KL can overweight fine-scale, near-terminal errors that do not translate proportionally into terminal
sample quality.

Accordingly, we optimize the schedule on a regularized SNR axis that compresses the high-SNR regime, where the
pathwise KL upper bound can be overly sensitive. For a parameter $\lambda>0$, define
\[
\gamma_{\mathrm{reg}}(\gamma)\;:=\;\frac{\gamma}{1+\lambda^2\gamma}.
\]
Since $\gamma_{\mathrm{reg}}(\gamma)$ is increasing and saturates as $\gamma\to\infty$, this transformation
smoothly downweights the influence of very large SNR values (near the terminal part of the reverse process), which
are precisely where the bound is typically loosest.

With $\eta_k:=\gamma_{\mathrm{reg}}(\gamma_k)$, we replace the unregularized objective by the surrogate
\begin{equation}
\label{eq:schedule-objective-reg}
\min_{\gamma_0<\gamma_1<\cdots<\gamma_K}\;
\sum_{k=1}^K\big(\eta_k-\eta_{k-1}\big)\,\mathcal L_{x_0}(\gamma_{k-1})
\end{equation}

\begin{table*}[t!]
\centering
\caption{ImageNet $256\times256$ metrics (FID, sFID, IS).}
\label{tab:imagenet_metrics}
\setlength{\tabcolsep}{4pt}
\begin{tabular}{l l ccc ccc}
\toprule
& & \multicolumn{3}{c}{NFE=10} & \multicolumn{3}{c}{NFE=20} \\
\cmidrule(lr){3-5}\cmidrule(lr){6-8}
Sampler & Schedule & FID $\downarrow$ & sFID $\downarrow$ & IS $\uparrow$ & FID $\downarrow$ & sFID $\downarrow$ & IS $\uparrow$ \\
\midrule
DDIM ($\eta=1$) & LAS              & \textbf{15.71} & \textbf{47.79} & \textbf{168.94} & \textbf{8.56}  & \textbf{15.33} & \textbf{282.89} \\
                &Time-uniform    & 25.06 & 68.56 & 111.78 & 9.44  & 22.84 & 255.74 \\ 
                & LogSNR           & 68.89 & 130.27 & 24.01 & 16.02 & 46.19 & 166.02 \\ 
                & EDM ($\rho = 7$) & 66.72 & 127.00 & 25.03 & 17.42 & 49.86 & 152.38\\ 
\midrule
SDE-DPM-Solver++ (2M)   & LAS              & \textbf{6.20}  & \textbf{10.59} & \textbf{273.76} & \textbf{6.67}  &\textbf{6.18} & 320.18 \\
                & Time-uniform      & 7.94  & 16.39 & 242.81 & 7.65  & 7.00 & \textbf{320.83} \\
                & LogSNR           & 10.54 & 29.60 & 191.28 & 7.15  & 8.56 & 308.36 \\ 
                & EDM ($\rho = 7$) & 9.91  & 27.23 & 199.77 & 7.36  & 8.735 & 296.43 \\ 
\midrule
DPM-Solver++ (2M)  & LAS         & \textbf{4.59}  & \textbf{5.74} & \textbf{263.69} & \textbf{4.84} & \textbf{5.37} & 283.09 \\
            & Time-uniform & 5.53  & 6.78 & 243.78  & 5.48 & 5.44 & \textbf{284.85} \\
            & LogSNR      & 4.95  & 6.93 & 251.08  & 4.98 & 5.47 & 280.89 \\
\bottomrule
\end{tabular}
\end{table*}

That is, we keep evaluating the model diagnostic at the \emph{true} SNR points $\gamma_{k-1}$, but we measure step
sizes on the \emph{regularized} axis via $\Delta\eta_k=\eta_k-\eta_{k-1}$, preventing near-terminal (high-SNR)
steps from dominating the optimization.

Once $\mathcal L_{x_0}(\gamma)$ is estimated on a finite candidate set of SNR values, the minimization in
\eqref{eq:schedule-objective-reg} becomes a minimum-cost selection of $K$ increasing grid points and can be solved
efficiently by a standard dynamic-programming shortest-path routine (Appendix ~\ref{app:schedule-optimization}). We call this schedule Loss-Adaptive Schedule (LAS).

\section{Experiments}
\label{sec:experiments}

We evaluate the proposed discretization schedule on both synthetic toy distributions and a large-scale image
generation benchmark. 

\subsection{Toy Examples: Gaussian Mixture Models}
We first consider controlled synthetic settings where the ground-truth data distribution is a Gaussian mixture
model (GMM). The corresponding
details are provided in Appendix~\ref{app:toy}.

\subsection{ImageNet $256\times256$ with Latent Diffusion}
We next evaluate on a real-world generative modeling task using latent diffusion models on ImageNet
$256\times256$ \cite{rombach2022high}. We use classifier guidance with scale $2$ and report Fr\'echet Inception
Distance (FID). We test two samplers:
(i) DDIM with $\eta=1$, which corresponds to the stochastic sampler consistent with our ``freezing-$m$''
discretization,
and (ii) SDE-DPM-Solver++(2M) and DPM-Solver++(2M)  second-order samplers \cite{lu2025dpm}.

For SDE-DPM-Solver++(2M) and DPM-Solver++(2M), we found the method to be sensitive to highly inhomogeneous step sizes due to its second-order
structure. In particular, second-order solvers are most stable when consecutive step sizes are \emph{comparable} (e.g., nearly
constant on the log-SNR axis): their local truncation error analysis and practical error cancellation across
adjacent steps can break down when the schedule is highly inhomogeneous, leading to instability and degraded
sample quality. To stabilize second-order sampling, we therefore encourage \emph{smooth} log-SNR step sizes.To stabilize the schedule for this sampler, we add an additional smoothness penalty to the schedule
optimization objective:
\begin{equation}
\alpha \sum_{k=2}^{K} (h_k - h_{k-1})^2,
\end{equation}
where $h_k := \log(\gamma_k/\gamma_{k-1})$ denotes the log-SNR ratio at step $k$.

\paragraph{Schedules and hyperparameters.}

We select $(\lambda,\alpha)$ using a small pilot budget of 1{,}000 generated samples and fix them thereafter.
All numbers reported in Table~\ref{tab:imagenet_metrics} are computed from an independent run of 50{,}000 generated
samples using the fixed hyperparameters. This mirrors standard practice for sampler hyperparameter tuning and
does not reuse the evaluation budget during selection.
Throughout all experiments we set
the SNR-axis regularization parameter to $\lambda=1.5$. Also,
we use $\alpha=12$ for all DPM-Solver experiments.

We compare the proposed LAS with three commonly used time-discretization schedules: Time-uniform, LogSNR, and EDM \cite{karras2022elucidating}.

Table~\ref{tab:imagenet_metrics} reports results on ImageNet $256\times256$ for number of function evaluations (NFE) 10 and 20.
Overall, LAS improves performance over the linear-time schedule for all three
samplers. The gains are particularly pronounced for the first-order method DDIM at low NFE (e.g., NFE$=10$),
which is consistent with the theory suggesting discretization effects are most visible in coarse
discretizations. We also observe improvements for SDE-DPM++(2M) and DPM-Solver++ (2M).

\section{Future work}

Our discretization bound is obtained by controlling the MMSE derivative via a general upper bound, and we expect this step is not tight in many regimes. In particular, under additional but still reasonable assumptions on the code distribution (e.g., separation properties of the code support), the MMSE regularity may admit sharper control, leading to improved constants or rates beyond the current $O(H^2/K)$ dependence. More broadly, we believe that exploiting local structure of the code distribution could yield sharper convergence guarantees.

On the scheduling side, LAS optimizes a regularized surrogate objective; the regularization we use is heuristic
and is motivated by the observation that the pathwise KL bound can become overly sensitive in the high-SNR (near-terminal) regime due to the data processing inequality used to relate pathwise KL to the terminal marginal.
A more refined analysis of this high-SNR data-processing gap may lead to a better-justified objective and further improvements to the resulting schedules. 
Finally, it would be valuable to extend our analysis and schedule design to higher-order samplers (e.g., second-order solvers such as DPM-Solver++), with the goal of deriving a principled schedule objective that accounts for higher-order discretization error and eliminates the need for the current heuristic step-size smoothing regularization.

\printbibliography

\newpage
\appendix
\onecolumn

\section{Proof of Proposition~\ref{prop:KL}}

We follow the proof of Theorem 10 in \cite{chen2023sampling}.

Define
\[
\mathcal E_t
:=\exp\!\left(-\int_0^t \delta_s^\top dB_s - \frac12\int_0^t \|\delta_s\|^2\,ds\right),\qquad t\in[0,T_\delta].
\]
Then $(\mathcal E_t)_{t\le T_\delta}$ is a nonnegative local $\mathbb P$-martingale.
Since global Novikov/Kazamaki may fail, we localize. Set
\[
\tau_n:=\inf\Bigl\{t\le T_\delta:\int_0^t \|\delta_s\|^2\,ds\ge n\Bigr\}\wedge T_\delta.
\]
Then $\int_0^{\tau_n}\|\delta_s\|^2ds\le n$ a.s., so Novikov holds on $[0,\tau_n]$ and
$\bigl(\mathcal E_{t\wedge\tau_n}\bigr)_{t\le T_\delta}$ is a true martingale; in particular,
\(
\mathbb E_{\mathbb P}[\mathcal E_{\tau_n}]=1.
\)

Define a probability measure $\tilde{\mathbb P}^n$ on $(\Omega,\mathcal F_{T_\delta})$ by
\[
\frac{d\tilde{\mathbb P}^n}{d\mathbb P}:=\mathcal E_{\tau_n}.
\]
By Girsanov's theorem applied to the stopped integrand $\delta\,\mathbf 1_{[0,\tau_n]}$, the process
\[
B^{(n)}_t:=B_t+\int_0^{t\wedge\tau_n}\delta_s\,ds
\]
is a $\tilde{\mathbb P}^n$-Brownian motion. Substituting $dB_t=dB^{(n)}_t-\delta_t\mathbf 1_{[0,\tau_n]}(t)\,dt$
into the reverse SDE yields that
\begin{equation*}
dY_t=\tilde\beta_t(Y_t)\,\mathbf 1_{[0,\tau_n]}(t)\,dt+\beta_t(Y_t)\,\mathbf 1_{(\tau_n,T_\delta]}(t)\,dt+dB^{(n)}_t.
\end{equation*}
In particular, up to time $\tau_n$ the drift is exactly $\tilde\beta$.

Since $\log\frac{d\tilde{\mathbb P}^n}{d\mathbb P}=\log \mathcal E_{\tau_n}$, we have
\[
\mathrm{KL}(\mathbb P\|\tilde{\mathbb P}^n)
=\mathbb E_{\mathbb P}\!\left[\log\frac{d\mathbb P}{d\tilde{\mathbb P}^n}\right]
=\mathbb E_{\mathbb P}\!\left[-\log \mathcal E_{\tau_n}\right]
=\mathbb E_{\mathbb P}\!\left[\int_0^{\tau_n}\delta_s^\top dB_s+\frac12\int_0^{\tau_n}\|\delta_s\|^2\,ds\right].
\]
The stochastic integral $\int_0^{\tau_n}\delta_s^\top dB_s$ is a (true) martingale with zero mean (it is stopped and square-integrable),
hence
\begin{equation*}
\mathrm{KL}(\mathbb P\|\tilde{\mathbb P}^n)
=\frac12\,\mathbb E_{\mathbb P}\!\left[\int_0^{\tau_n}\|\delta_s\|^2\,ds\right]
\le \frac12\,\mathbb E_{\mathbb P}\!\left[\int_0^{T_\delta}\|\delta_s\|^2\,ds\right].
\end{equation*}

We now consider a coupling of $(\tilde{\mathbb P}^n)_{n\in\mathbb N}$ and $\tilde{\mathbb P}$: stochastic processes $(\tilde \xi^n)_{n\in\mathbb N}$ and $\tilde \xi$ on $[0,T_\delta]$ driven by a single Brownian motion $\bar B$ such that
\begin{equation*}
d\tilde\xi^n_t=\tilde\beta_t(\tilde\xi^n_t)\,\mathbf 1_{[0,\tau_n]}(t)\,dt+\beta_t(\tilde\xi^n_t)\,\mathbf 1_{(\tau_n,T_\delta]}(t)\,dt+d\bar B_t,
\end{equation*}
and
\begin{equation*}
d\tilde\xi_t = \tilde\beta_t(\tilde\xi_t)\,dt + d\bar B_t,
\end{equation*}
with $\tilde\xi^n_0=\tilde\xi_0\stackrel{d}{=} X_{T}$ a.s. for all $n$.

Note that the distribution of $\tilde \xi^n$ is $\tilde{\mathbb P}^n$ for all $n$ and the distribution of $\tilde \xi$ is $\tilde{\mathbb P}$. Also, for any $n$, we have $\tilde\xi^n_t=\tilde\xi_t$ a.s. for every $t\in[0,\tau_n]$.

Fix $\varepsilon\in(0,T_\delta)$ and consider the truncation map on path space
\[
\pi_\varepsilon:C([0,T_\delta];\mathbb R^d)\to C([0,T_\delta];\mathbb R^d),
\qquad
(\pi_\varepsilon(\omega))(t):=\omega\bigl(t\wedge (T_\delta-\varepsilon)\bigr).
\]

Since the square-integrability condition \eqref{eq:finite-energy-under-P}
implies $\tau_n\uparrow T_\delta$ $\mathbb{P}$-a.s., it follows that $\pi_\varepsilon(\tilde\xi^n)\to\pi_\varepsilon(\tilde\xi)$ almost surely,
uniformly on $[0,T_\delta]$. Hence
\[
(\pi_\varepsilon)_\#\tilde{\mathbb P}^n \Rightarrow (\pi_\varepsilon)_\#\tilde{\mathbb P}
\qquad\text{weakly on }C([0,T_\delta];\mathbb R^d).
\]
Therefore, by the lower semicontinuity of $\mathrm{KL}$ under weak convergence and the data-processing inequality,
\begin{align*}
\mathrm{KL}\bigl((\pi_\varepsilon)_\#\mathbb P\,\big\|\,(\pi_\varepsilon)_\#\tilde{\mathbb P}\bigr)
&\le \liminf_{n\to\infty}\mathrm{KL}\bigl((\pi_\varepsilon)_\#\mathbb P\,\big\|\,(\pi_\varepsilon)_\#\tilde{\mathbb P}^n\bigr) \\
&\le \liminf_{n\to\infty}\mathrm{KL}\bigl(\mathbb P\,\big\|\,\tilde{\mathbb P}^n\bigr)
\le \frac12\,\mathbb E_{\mathbb P}\!\left[\int_0^{T_\delta}\|\delta_s\|^2\,ds\right].
\end{align*}
Finally, since $\pi_\varepsilon(\omega)\to\omega$ uniformly as $\varepsilon\downarrow 0$,
\[
\mathrm{KL}(\mathbb P\|\tilde{\mathbb P})
=\lim_{\varepsilon\downarrow 0}\mathrm{KL}\bigl((\pi_\varepsilon)_\#\mathbb P\,\big\|\,(\pi_\varepsilon)_\#\tilde{\mathbb P}\bigr)
\le \frac12\,\mathbb E_{\mathbb P}\!\left[\int_0^{T_\delta}\|\delta_s\|^2\,ds\right].
\]

\section{Proof of Proposition~\ref{thm:disc-mmse}}

\begin{proof}
The forward process is $X_t = Z + W_t$ for $t \in [0, T]$, where $Z \sim p$ is independent of the standard Brownian motion $(W_t)_{t \geq 0}$. Parameterize observations by the signal-to-noise ratio $\gamma = 1/t > 0$, so the observation at SNR $\gamma$ is $X_{1/\gamma} = Z + W_{1/\gamma}$.

Define the increasing filtration $(\mathcal{F}_\gamma)_{\gamma > 0}$ by
\[
\mathcal{F}_\gamma := \sigma\bigl( X_{1/\gamma'} : 0 < \gamma' \leq \gamma \bigr)
\]
(equivalently, the sigma-algebra generated by $\{W_u : u \geq 1/\gamma\}$). Higher $\gamma$ corresponds to lower noise variance $1/\gamma$, so the filtration is increasing: $\gamma_1 < \gamma_2$ implies $\mathcal{F}_{\gamma_1} \subset \mathcal{F}_{\gamma_2}$.

Let
\[
M_\gamma := \mathbb{E}[Z \mid \mathcal{F}_\gamma] = \mathbb{E}\bigl[ Z \mid X_{1/\gamma} \bigr] = m_{1/\gamma}(X_{1/\gamma})
\]
be the posterior mean (Bayes-optimal denoiser) at SNR $\gamma$.

By the tower property, for $0 < \gamma_1 < \gamma_2$,
\[
M_{\gamma_1} = \mathbb{E}\bigl[ Z \mid \mathcal{F}_{\gamma_1} \bigr] = \mathbb{E} \bigl[ \mathbb{E}[Z \mid \mathcal{F}_{\gamma_2}] \bigm| \mathcal{F}_{\gamma_1} \bigr] = \mathbb{E}[ M_{\gamma_2} \mid \mathcal{F}_{\gamma_1} ].
\]
Thus, $(M_\gamma)_{\gamma > 0}$ is a vector-valued martingale with respect to $(\mathcal{F}_\gamma)$---the Doob martingale associated with the integrable target $Z$.

The MMSE at SNR $\gamma$ is
\[
\mathrm{mmse}(\gamma) := \mathbb{E} \bigl[ \|Z - M_\gamma\|_2^2 \bigr].
\]
Since $\mathcal{F}_{\gamma_{k-1}} \subset \mathcal{F}_\gamma$ for $\gamma > \gamma_{k-1}$, the corresponding subspaces of $\mathcal{F}_\gamma$-measurable random vectors are nested. The error $Z - M_\gamma$ is orthogonal (in $L^2(\mathbb{P})$) to all $\mathcal{F}_\gamma$-measurable functions, and in particular to the martingale increment $M_\gamma - M_{\gamma_{k-1}}$.

Decompose
\[
Z - M_{\gamma_{k-1}} = (Z - M_\gamma) + (M_\gamma - M_{\gamma_{k-1}}).
\]
The cross-term vanishes:
\[
\mathbb{E} \bigl[ (Z - M_\gamma)^\top (M_\gamma - M_{\gamma_{k-1}}) \bigr] = 0,
\]
so by the Pythagorean theorem,
\[
\mathbb{E} \bigl[ \|Z - M_{\gamma_{k-1}}\|_2^2 \bigr] = \mathbb{E} \bigl[ \|Z - M_\gamma\|_2^2 \bigr] + \mathbb{E} \bigl[ \|M_\gamma - M_{\gamma_{k-1}}\|_2^2 \bigr].
\]
Hence,
\begin{equation}\label{eq:mmsedif}
\mathbb{E} \bigl[ \|M_\gamma - M_{\gamma_{k-1}}\|_2^2 \bigr] = \mathrm{mmse}(\gamma_{k-1}) - \mathrm{mmse}(\gamma). 
\end{equation}

Now express $\mathcal{E}_{\mathrm{disc}}$. By definition,
\[
\mathcal{E}_{\mathrm{disc}} = \sum_{k=1}^K \mathbb{E} \left[ \int_{s_{k-1}}^{s_k} \| m_{T-s}(Y_s) - m_{T-s_{k-1}}(Y_{s_{k-1}}) \|_2^2 \frac{ds}{(T-s)^2} \right].
\]
The change of variables $\gamma = 1/(T-s)$ gives $d\gamma = ds / (T-s)^2$, and as $s$ runs from $s_{k-1}$ to $s_k$, $\gamma$ runs from $\gamma_{k-1}$ to $\gamma_k$ (increasing). Since we also have $(Y_s)_{s\in[0,T_\delta]} = (X_{T-s})_{s\in[0,T_\delta]}$\,, each summand becomes
\[
\mathbb{E} \left[ \int_{s_{k-1}}^{s_k} \| m_{T-s}(Y_s) - m_{T-s_{k-1}}(Y_{s_{k-1}}) \|_2^2 \frac{ds}{(T-s)^2} \right] = \mathbb{E} \left[\int_{\gamma_{k-1}}^{\gamma_k} \| M_\gamma - M_{\gamma_{k-1}} \|_2^2 \, d\gamma \right].
\]
Interchanging the order of expectation and integration (justified by Tonelli's theorem),
\[
\mathbb{E} \left[ \int_{\gamma_{k-1}}^{\gamma_k} \| M_\gamma - M_{\gamma_{k-1}} \|_2^2 \, d\gamma \right] = \int_{\gamma_{k-1}}^{\gamma_k} \mathbb{E} \bigl[ \| M_\gamma - M_{\gamma_{k-1}} \|_2^2 \bigr] \, d\gamma.
\]
Applying \eqref{eq:mmsedif}, this becomes 
\[
\int_{\gamma_{k-1}}^{\gamma_k} \bigl( \mathrm{mmse}(\gamma_{k-1}) - \mathrm{mmse}(\gamma) \bigr) \, d\gamma.
\]
Summing over $k = 1, \dots, K$ yields the desired expression for $\mathcal{E}_{\mathrm{disc}}$.

This expresses the discretization error as the total expected quadratic variation of the missed martingale increments when the denoiser is held constant within each reverse step, instead of following the continuous Doob martingale $M_\gamma$.
\end{proof}

\section{Proof of Theorem~\ref{thm:deriv_mmse_bound}}
We now show how the derivative bound on the MMSE in Theorem~\ref{thm:deriv_mmse_bound} is obtained. The argument proceeds in three steps. First, we express the derivative of the MMSE in terms of a conditional covariance. Second, we control the trace of the squared covariance by a fourth-moment quantity. Finally, we bound this fourth moment using Shannon entropy of the target distribution under Assumption~\ref{assumption}.



Fix $t\in[0,T]$. According to the forward process, the joint law of $(Z,X_t)$ is 
\begin{equation*}
\mathbb{P}\big(Z=z,\ X_t\in dx\big)
=
p(z)\,p_t(x\mid z)\,dx,
\qquad z\in\mathcal C,
\end{equation*}
where $p_{t}(x\mid z)$ is the probability density function of Gaussian distribution 
\(\mathcal{N}(z,tI_d)\). We compute the posterior 
\begin{equation}\label{eq:post}
    p^{\mathrm{post}}_{t}(z\mid x)=\frac{p(z)p_{t}(x\mid z)}{p_{X_t}(x)}\,,
\end{equation}
where $p_{X_t}(x) = \sum_{u\in\mathcal{C}}p(u)p_{t}(x\mid u)$ is the probability density function of $X_t$.

Introduce a new random variable $Z'$. We specify the joint law of $(Z,X_t,Z')$ by 
\begin{equation*}
\mathbb{P}\big(Z=z,\ X_t\in dx,\ Z'=z'\big)
=
p(z)\,p_t(x\mid z)\,r_t(z'\mid x,z)\,dx,
\qquad z,z'\in\mathcal C.
\end{equation*}
where \(r_{t}(z'\mid x,z) = p^{\mathrm{post}}_{t}(z'\mid x)\). 

Note that conditional on $X_t$, the random variable $Z'$ can be considered as a posterior draw independent of $Z$. Indeed,
\begin{align*}
p_{Z,Z'\mid X_t}(z,z'\mid x)
&=
\frac{p(z)\,p_t(x\mid z)\,r_t(z'\mid x,z)}{p_{X_t}(x)}\\
&=
\frac{p(z)\,p_t(x\mid z)}{p_{X_t}(x)}\, p^{\mathrm{post}}_{t}(z'\mid x) \\
&=
p^{\mathrm{post}}_{t}(z\mid x)\,p^{\mathrm{post}}_{t}(z'\mid x).
\end{align*}

Define the kernel 
\begin{equation*}
q_{t,z}(z')
:= \mathbb{P}(Z'=z'\mid Z=z)
=\int_{\R^d} p^{\mathrm{post}}_t(z'\mid x)\, p_t(x\mid z)\,dx,
\qquad z,z'\in\mathcal C.
\end{equation*}


The proof for Theorem~\ref{thm:deriv_mmse_bound} starts from a recent identity from information theory, which relates the derivative of the MMSE along the Gaussian channel to the conditional covariance of the posterior. The following proposition is a reformulation of a result from \cite{nguyen2024beyond}, but we give an alternative proof via martingale representation.

\begin{proposition}
\label{prop:mmse-deriv-cov}

The MMSE function
\[
\mathrm{mmse}(\gamma)=\mathbb E\big[\|Z-\mathbb E[Z\mid X_{1/\gamma}]\|_2^2\big]
\]
is differentiable and monotonically decreasing for $\gamma>0$, and
\begin{equation}
|\mathrm{mmse}'(\gamma)|
=-\mathrm{mmse}'(\gamma)
=\mathbb E\!\left[
\operatorname{tr}\!\left(\operatorname{Cov}(Z\mid X_{1/\gamma})^2\right)
\right]
=\mathbb E\!\left[
\operatorname{tr}\!\left(\operatorname{Cov}(Z\mid X_t)^2\right)
\right],
\label{eq:mmse-deriv-cov}
\end{equation}
where $t=1/\gamma$.
\end{proposition}

\begin{proof}

Recall we defined the reverse-time process $(Y_s)_{s\in[0,T]}$ by $Y_s := X_{T-s}$ for all $s\in[0,T]$. Let $\mathcal F_s:=\sigma(Y_u:0\le u\le s)$ be the natural filtration of $(Y_s)_{s\in[0,T]}$. Define 
\[
M_s:=\mathbb E[Y_T\mid \mathcal F_s].
\]

Since $Y$ is Markov, $\mathbb E[Y_T\mid \mathcal F_s]=\mathbb E[Y_T\mid Y_s]$, hence
$M_s=u(s,Y_s)$ where $u(s,y):=\mathbb E[Y_T\mid Y_s=y]$.
Since $(M_s)_{s\in[0,T]}$ is a square-integrable
continuous martingale, and moreover, the diffusion coefficient of the reverse SDE is the identity,
we have
\[
dM_s = \nabla_y u(s,Y_s)\, d B_s
\]
by It\^o's formula. 
Since $M_s$ is a continuous martingale, It\^o's formula gives
\[
d\|M_s\|_2^2 = 2\langle M_s, dM_s\rangle + d\langle M\rangle_s,
\]
and taking expectation kills the martingale term, hence
\begin{equation*}
\frac{d}{ds}\mathbb E\|M_s\|_2^2
=\mathbb E\!\left[\frac{d}{ds}\langle M\rangle_s\right]
=\mathbb E\!\left[\|\nabla_y u(s,Y_s)\|_F^2\right],
\end{equation*}
where $\|\cdot\|_F$ is the Frobenius norm. 

Define
\[
\mathrm{mmse}_{\rm rev}(s):=\mathbb E\big[\|Y_T-\mathbb E[Y_T\mid \mathcal F_s]\|_2^2\big]
=\mathbb E\big[\|Y_T-M_s\|_2^2\big].
\]
Using orthogonality of conditional expectation,
\[
\mathrm{mmse}_{\rm rev}(s)=\mathbb E\|Y_T\|_2^2-\mathbb E\|M_s\|_2^2.
\]
Differentiating in $s$ gives
\begin{equation}
\mathrm{mmse}_{\rm rev}'(s)
=-\mathbb E\!\left[\|\nabla_y u(s,Y_s)\|_F^2\right].
\label{eq:mmse-rev-deriv}
\end{equation}

Recall that $Y_s=X_{T-s}$ and $Y_T=Z$. Thus, $u(s,y)=\mathbb E[Z\mid X_{T-s}=y]$. Denote $m_t(y):=\mathbb E[Z\mid X_t=y]$ and
$\Sigma_t(y):=\operatorname{Cov}(Z\mid X_t=y)$.
A standard differentiation-under-the-integral calculation for the Gaussian channel
(see e.g.\ Tweedie-type identities) yields the matrix Jacobian identity
\begin{equation}
\nabla_y m_t(y)=\frac{1}{t}\,\Sigma_t(y).
\label{eq:jacobian-cov}
\end{equation}
(Quick derivation: $m_t(y)=\frac{\int z\,p_Z(z)\varphi_t(y-z)\,dz}{\int p_Z(z)\varphi_t(y-z)\,dz}$,
differentiate using $\nabla_y\varphi_t(y-z)=-(y-z)\varphi_t(y-z)/t$, and simplify to obtain
$\nabla_y m_t(y)=\frac{1}{t}\big(\mathbb E[ZZ^\top\!\mid X_t=y]-m_t(y)m_t(y)^\top\big)$.)

Combining \eqref{eq:mmse-rev-deriv}--\eqref{eq:jacobian-cov} and using
$\|\Sigma\|_F^2=\operatorname{tr}(\Sigma^2)$ for symmetric $\Sigma$,
\[
\mathrm{mmse}_{\rm rev}'(s)
=-\mathbb E\!\left[\left\|\frac{1}{t}\Sigma_t(X_t)\right\|_F^2\right]
=-\frac{1}{t^2}\mathbb E\!\left[\operatorname{tr}\!\big(\operatorname{Cov}(Z\mid X_t)^2\big)\right],
\]
where $t=T-s$.



Note that since $Y_s=X_t$ and $Y_T=Z$, we have $\mathrm{mmse}(\gamma)=\mathrm{mmse}_{\rm rev}(s)$, where $\gamma=1/t=1/(T-s)$. By the chain rule,
\[
\mathrm{mmse}'(\gamma)=\mathrm{mmse}_{\rm rev}'(s)\cdot\frac{ds}{d\gamma}
= \left(-\gamma^2\,\mathbb E\!\left[\operatorname{tr}\!\big(\operatorname{Cov}(Z\mid X_{1/\gamma})^2\big)\right]\right)\cdot\left(\frac{1}{\gamma^2}\right),
\]
hence \eqref{eq:mmse-deriv-cov}. 
In particular, $\mathrm{mmse}'(\gamma)\le 0$ and
$\mathrm{mmse}$ is monotonically decreasing.
\end{proof}

Identity \eqref{eq:mmse-deriv-cov} reduces the problem of bounding $\mathrm{mmse}'(\gamma)$ to controlling the squared conditional covariance of the posterior distribution of $Z$ given a noisy observation.

To control the right-hand side of \eqref{eq:mmse-deriv-cov}, we bound the trace of the squared covariance matrix by a fourth moment using the following probabilistic lemma.

\begin{lemma}\label{lem:cov2-4th}
Let \(Y\) be an \(\R^d\)-valued random vector with mean \(m:=\mathbb E[Y]\) and covariance \(\Sigma:=\operatorname{Cov}(Y)\). Then for any \(a\in\R^d\),
\begin{equation*}\label{eq:cov2-4th-lemma}
\operatorname{tr}(\Sigma^2) \le \mathbb E\big[\|Y-a\|_2^4\big].
\end{equation*}
\end{lemma}

\begin{proof}
Since the covariance matrix $\Sigma$ is symmetric 
positive semidefinite, its eigenvalues $\lambda_1,\dots,\lambda_d$ 
are all nonnegative. This implies
\[
\operatorname{tr}(\Sigma^2)=\sum_{i=1}^d \lambda_i^2 
\;\le\; \left(\sum_{i=1}^d \lambda_i\right)^2 
= (\operatorname{tr}\Sigma)^2.
\]
Then we evaluate
\[
\operatorname{tr}\Sigma 
= \operatorname{tr}\bigl(\mathbb{E}[(Y-m)(Y-m)^\top]\bigr)
= \mathbb{E}\,\operatorname{tr}\bigl((Y-m)(Y-m)^\top\bigr)
= \mathbb{E}\|Y-m\|_2^2,
\]
where we used linearity of trace and expectation, and the identity 
$\operatorname{tr}(vv^\top)=\|v\|_2^2$ for any vector $v$.

Since $m=\mathbb{E}Y$ is the unique minimizer of 
the function $a\mapsto \mathbb{E}\|Y-a\|_2^2$, 
\[
\mathbb{E}\|Y-m\|_2^2 
\;\le\; \mathbb{E}\|Y-a\|_2^2
\qquad\text{for every } a\in\mathbb{R}^d.
\]

Combining these and finally applying the Cauchy--Schwarz 
inequality yields for every $a\in\mathbb{R}^d$, 
\[
\operatorname{tr}(\Sigma^2)
\;\le\;
(\operatorname{tr}\Sigma)^2
\;=\;
\bigl(\mathbb{E}\|Y-m\|_2^2\bigr)^2
\;\le\;
\bigl(\mathbb{E}\|Y-a\|_2^2\bigr)^2
\;\le\;
\mathbb{E}\|Y-a\|_2^4.
\]
\end{proof}

Applying Lemma~\ref{lem:cov2-4th} yields the following bound for the trace of the squared covariance matrix.

\begin{proposition}
\label{prop:cond-typical}

For each \(z \in \mathcal{C}\),
\[
\mathbb{E}\bigl[ \operatorname{tr}\bigl( \Cov(Z' \mid X_t)^2 \bigr) \mid Z = z \bigr]
\le \mathbb{E}_{Z' \sim q_{t,z}} \bigl[ \|Z' - z\|_2^4 \bigr],
\]
where \(q_{t,z}(z') := \int p^{\mathrm{post}}_t(z' \mid x)\, p_t(x \mid z)\, dx\).

Moreover, 
\[
\mathbb{E}\bigl[ \operatorname{tr}\bigl( \Cov(Z' \mid X_t)^2 \bigr) \bigr]
\le \mathbb{E}\bigl[ \|Z' - Z^*\|_2^4 \bigr].
\]
\end{proposition}

\begin{proof}

For every \(z \in \mathcal{C}\), we have
\begin{align}
\mathbb{E}\bigl[ \operatorname{tr}\bigl( \Cov(Z' \mid X_t)^2 \bigr) \mid Z = z \bigr]
&= \int \operatorname{tr}\bigl( \Cov_{\xi\sim p^{\mathrm{post}}_t(\cdot \mid x)}(\xi)^2 \bigr) \, p_t(x \mid z) \, dx \notag\\
&\le \int \mathbb{E}_{\xi\sim p^{\mathrm{post}}_t(\cdot \mid x)}\bigl[ \|\xi - z\|_2^4 \bigr] \,  p_t(x \mid z) \, dx \label{eq:line2}\\
&= \int \int \|\xi - z\|_2^4 \, p^{\mathrm{post}}_t(\xi \mid x)\, p_t(x \mid z) \, d\xi \, dx \notag\\
&= \int \int \|\xi - z\|_2^4 \, q_{t,z} (\xi) \, d\xi  \label{eq:line4}\\
&= \mathbb{E}_{Z' \sim q_{t,z}} \bigl[ \|Z' - z\|_2^4 \bigr]\notag\,,
\end{align}
where for \eqref{eq:line2}, we applied Lemma~\ref{lem:cov2-4th} with anchor point \(a = z\) (which is constant given the outer conditioning on \(Z = z\)) to the conditional distributions; and \eqref{eq:line4} follows from the Tonelli's theorem.

Finally, taking expectation over \(Z \sim p(\cdot)\) and using the definition $q_{t,z}(z')
= \mathbb{P}(Z'=z'\mid Z=z)$ gives
\begin{align*}
\mathbb{E}\bigl[ \operatorname{tr}\bigl( \Cov(Z' \mid X_t)^2 \bigr) \bigr]
&= \sum_{z\in\mathcal C} p(z)\, \mathbb{E}\bigl[ \operatorname{tr}\bigl( \Cov(Z' \mid X_t)^2 \bigr) \,\big|\, Z=z \bigr]\\
&\le \sum_{z\in\mathcal C} p(z)\, \mathbb{E}_{Z' \sim q_{t,z}} \bigl[ \|Z' - z\|_2^4 \bigr]\\
&= \mathbb{E}_{Z} \Bigl[ \mathbb{E}\bigl[ \|Z' - Z\|_2^4 \,\big|\, Z \bigr] \Bigr] \\
&= \mathbb{E}\bigl[ \|Z' - Z\|_2^4 \bigr].
\end{align*}
\end{proof}

Combined with Proposition~\ref{prop:mmse-deriv-cov}, this shows that controlling $|\mathrm{mmse}'(\gamma)|$ reduces to bounding a fourth moment of the posterior fluctuations.

The final step is to bound $\mathbb E[\|Z'-Z\|_2^4]$ using an information-theoretic quantity of the target distribution. As an intermediate tool for the proof, we define R\'enyi entropy of order $1/2$.

\begin{definition}[R\'enyi entropy of order $1/2$]\label{def:renyi}
For a discrete distribution $p$ supported on $\mathcal C$, define
\[
H_{1/2}
:= \frac{1}{1-\tfrac12}\log\sum_{z\in\mathcal C}p(z)^{1/2}
= 2\log\sum_{z\in\mathcal C}\sqrt{p(z)}\,.
\]
\end{definition}

\begin{proposition}\label{thm:4th-moment-renyi}
There exists a universal constant $C_4>0$ such that for all $t>0$,
\begin{equation}\label{eq:4th-moment-renyi}
\mathbb E\big[\|Z'-Z\|_2^4\big]
\le C_4\,t^2 H_{1/2}^2.
\end{equation}
\end{proposition}

\begin{proof}
The inequality is trivial for $H_{1/2}=\infty$. We only need to prove for $H_{1/2}<\infty$.


For any $z'\neq z$, keeping only two terms in the denominator of \eqref{eq:post} gives
\[
p^{\mathrm{post}}_{t}(z'\mid x)\le \frac{p(z')p_{t}(x\mid z')}{p(z')p_{t}(x\mid z')+p(z)p_{t}(x\mid z)}.
\]
Apply the AM–GM inequality to the denominator yields
\[
p(z')p_{t}(x\mid z')+p(z)p_{t}(x\mid z)\ge 2\sqrt{p(z')p_{t}(x\mid z')p(z)p_{t}(x\mid z)}\,.
\]
Hence,
\[
p^{\mathrm{post}}_{t}(z'\mid x)\le \frac{p(z')p_{t}(x\mid z')}{2\sqrt{p(z')p_{t}(x\mid z')p(z)p_{t}(x\mid z)}}
= \frac{1}{2}\sqrt{\frac{p(z')p_{t}(x\mid z')}{p(z)p_{t}(x\mid z)}}\,,
\]
and
\[
q_{t,z}(z')=\int p^{\mathrm{post}}_{t}(z'\mid x)\,p_{t}(x\mid z)\,dx
\le \tfrac12\sqrt{\frac{p(z')}{p(z)}}
\int \sqrt{p_{t}(x\mid z')p_{t}(x\mid z)}\,dx.
\]
For Gaussian kernels \(p_{t}(x\mid z)\) and \( p_{t}(x\mid z')\), one computes
\[
\int \sqrt{p_{t}(x\mid z')p_{t}(x\mid z)}\,dx
= \exp\Big(-\frac{\|z'-z\|_2^2}{8t}\Big).
\]
Hence,
\begin{equation*}\label{eq:q-basic}
q_{t,z}(z')\le \tfrac12\sqrt{\frac{p(z')}{p(z)}}
\exp\Big(-\frac{\|z'-z\|_2^2}{8t}\Big).
\end{equation*}

Denote \(R := \|Z'-Z\|_2.\)
For any $r\ge0$,
\begin{align*}
\mathbb P(R> r\mid Z=z)
&= \sum_{\|z'-z\|_2 > r} q_{t,z}(z') \\
&\le \frac{1}{2\sqrt{p(z)}}
\sum_{\|z'-z\|_2 > r} \sqrt{p(z')}\,
\exp\Big(-\frac{\|z'-z\|_2^2}{8t}\Big) \\
&\le \frac{1}{2\sqrt{p(z)}} e^{-r^2/(8t)} \sum_{z'\in\mathcal C}\sqrt{p(z')}.
\end{align*}
Define
\[
S := \sum_{z\in\mathcal C}\sqrt{p(z)}.
\]
By definition of $H_{1/2}$ we have $S^2 = e^{H_{1/2}}$.
Averaging over $Z\sim p(\cdot)$,
\begin{align}
\mathbb P(R > r)
&= \sum_{z\in\mathcal C} p(z)\,\mathbb P(R> r\mid Z=z) \notag\\
&\le \frac{S}{2}e^{-r^2/(8t)}\sum_{z\in\mathcal C}\sqrt{p(z)} \notag\\
&= \frac{S^2}{2}e^{-r^2/(8t)}\notag\\
&= \frac{1}{2}\exp\Big(H_{1/2}-\frac{r^2}{8t}\Big).
\label{eq:R-tail-renyi}
\end{align}

Let
\[
r_0 := \sqrt{Ct\,H_{1/2}},
\]
where $C\ge 0$ is a constant to be chosen later.
Decompose
\[
\mathbb E[R^4]
= \mathbb E\big[R^4\indic\{R\le r_0\}\big]
+ \mathbb E\big[R^4\indic\{R> r_0\}\big].
\]

On the set $\{R\le r_0\}$ we simply use $R^4\le r_0^4$, hence
\begin{equation}\label{eq:R4-inner}
\mathbb E\big[R^4\indic\{R\le r_0\}\big]
\le r_0^4
= C^2 t^2 H_{1/2}^2.
\end{equation}

For the tail part, we use \eqref{eq:R-tail-renyi}: 
\begin{align*}
\mathbb E\big[R^4\indic\{R> r_0\}\big]
&= \int_{r_0}^\infty 4r^3 \mathbb P(R > r)\,dr + r_0^4 \mathbb P(R > r_0)\\
&\le 2 e^{H_{1/2}}\int_{r_0}^\infty r^3 e^{-r^2/(8t)}\,dr + \frac{C^2 t^2 H_{1/2}^2}{2}\exp\Big(H_{1/2}-\frac{r_0^2}{8t}\Big).
\end{align*}
Evaluating the integral and plugging in $r_0 = \sqrt{Ct\,H_{1/2}}$\,, we get
\begin{equation}\label{eq:R4-outer-raw}
\mathbb E\big[R^4\indic\{R> r_0\}\big]
\le
\left[64 t^2\Big(\frac{C}{8}H_{1/2}+1\Big)+ \frac{C^2 t^2 H_{1/2}^2}{2}\right]
\exp\Big(\Big(1-\frac{C}{8}\Big)H_{1/2}\Big)
.
\end{equation}


Combining \eqref{eq:R4-inner} and \eqref{eq:R4-outer-raw} yields
\[
\mathbb E[R^4]
\le C^2 t^2 H_{1/2}^2
+ \left[64 t^2\Big(\frac{C}{8}H_{1/2}+1\Big)+ \frac{C^2 t^2 H_{1/2}^2}{2}\right]
\exp\Big(\Big(1-\frac{C}{8}\Big)H_{1/2}\Big)\,.
\]

This holds for all $C\ge 0$. In particular, for $C=8$, we have 
\[
\mathbb E[R^4]
\le 96 t^2 H_{1/2}^2
+ 64 t^2\Big(H_{1/2}+1\Big)\,.
\]

This holds for all $t\ge 0$, so there exists a universal constant $C_4>0$ 
such that for all $H_{1/2}\ge0$ and $t\ge 0$,
\[\mathbb E[R^4]
\le C_4\,t^2 H_{1/2}^2\,.\]
\end{proof}

Finally, we bound R\'enyi entropy with Shannon entropy up to a constant, under a sub-exponential assumption on information content.

\begin{proposition}\label{prop:4th-moment-shannon}
Under Assumption~\ref{assumption}, 
\begin{equation}\label{eq:renyi-vs-shannon}
H_{1/2} \le H + \frac{\nu^2}{2},
\end{equation}
and there exists a constant $C_5>0$ (depending only on
the universal constant $C_4$ from Proposition~\ref{thm:4th-moment-renyi}
and on $\nu^2$) such that for all $t>0$,
\begin{equation}\label{eq:4th-moment-shannon}
\mathbb E\big[\|Z'-Z\|_2^4\big]
\le C_5\,t^2 H^2.
\end{equation}
\end{proposition}

\begin{proof}
Recall
\[
S := \sum_{z\in\mathcal C}\sqrt{p(z)}
= \mathbb E\big[e^{\info(Z)/2}\big]
= e^{H/2}\,\mathbb E\big[e^{(\info(Z)-H)/2}\big]\,.
\]

By \eqref{eq:SE-again} with $\lambda=\tfrac12$ (which is allowed
because $b\le 2$ implies $1/2\le 1/b$),
\[
\mathbb E\big[e^{(\info(Z)-H)/2}\big]
\le e^{\nu^2/4},
\]
and therefore
\[
S \le e^{H/2} e^{\nu^2/4}.
\]
Taking logarithms and recalling
$H_{1/2}=2\log S$ yields
\[
H_{1/2}
= 2\log S
\le 2\Big(\frac{H}{2} + \frac{\nu^2}{4}\Big)
= H + \frac{\nu^2}{2},
\]
which is \eqref{eq:renyi-vs-shannon}.

Substituting \eqref{eq:renyi-vs-shannon} into
\eqref{eq:4th-moment-renyi} from Proposition~\ref{thm:4th-moment-renyi}
gives
\[
\mathbb E\big[\|Z'-Z\|_2^4\big]
\le C_4\,t^2 \bigl(H+\tfrac{\nu^2}{2}\bigr)^2
\le C_5\,t^2 H^2,
\]
for some constant $C_5>0$ depending only on $C_4$ and $\nu^2$.
This proves \eqref{eq:4th-moment-shannon}.
\end{proof}

Combining Propositions~\ref{prop:mmse-deriv-cov}, \ref{prop:cond-typical} and \ref{prop:4th-moment-shannon}, we obtain a dimension-free bound on $|\mathrm{mmse}'(\gamma)|$, which exactly leads to Theorem~\ref{thm:deriv_mmse_bound}.
\section{Proof of Proposition~ \ref{prop:Eapprox-logsnr}}

\begin{proof}
On $(s_{k-1},s_k]$, the term
$\bigl\|m_{T-s_{k-1}}(Y_{s_{k-1}})-\hat m_{T-s_{k-1}}(Y_{s_{k-1}})\bigr\|_2^2$
is $\mathcal F_{s_{k-1}}$-measurable, and
\[
\int_{s_{k-1}}^{s_k}\frac{ds}{(T-s)^2}
=
\Big[\frac{1}{T-s}\Big]_{s_{k-1}}^{s_k}
=
\gamma_k-\gamma_{k-1}.
\]
Using $Y_{s_{k-1}}=X_{T-s_{k-1}}=X_{1/\gamma_{k-1}}$ and
\[
m_{1/\gamma}(x)-\hat m_{1/\gamma}(x)
=
\frac{1}{\sqrt{\gamma}}\bigl(\hat\varepsilon_\gamma(x)-\varepsilon_\gamma^\star(x)\bigr),
\]
we obtain
\[
\mathcal E_{\mathrm{apx}}
=
\sum_{k=1}^K
(\gamma_k-\gamma_{k-1})\cdot
\E\!\left[\frac{1}{\gamma_{k-1}}
\big\|\varepsilon_{\gamma_{k-1}}^\star(X_{1/\gamma_{k-1}})
-\hat\varepsilon_{\gamma_{k-1}}(X_{1/\gamma_{k-1}})\big\|_2^2\right],
\]
which implies \eqref{eq:Eapprox-logsnr} under the SNR grid \eqref{eq:gamma-geometric}.
\end{proof}

\section{Proof of Proposition~ \ref{prop:disc+approx}}

\begin{proof}
By Proposition~\ref{thm:disc-mmse},
\begin{align}
\mathcal E_{\mathrm{disc}}
&=\sum_{k=1}^K \int_{\gamma_{k-1}}^{\gamma_k}
\big(\mathrm{mmse}(\gamma_{k-1})-\mathrm{mmse}(\gamma)\big)\,d\gamma \nonumber\\
&=\sum_{k=1}^K(\gamma_k-\gamma_{k-1})\,\mathrm{mmse}(\gamma_{k-1})
\;-\;\int_{\gamma_0}^{\gamma_K}\mathrm{mmse}(\gamma)\,d\gamma.
\label{eq:disc-expand}
\end{align}

Next, from the definition of $\mathcal E_{\mathrm{apx}}$,
the integrand does not depend on $s$ except through $(T-s)^{-2}$, hence
\begin{align}
\mathcal E_{\mathrm{apx}}
&=\sum_{k=1}^K
\E\!\left[
\big\|m_{T-s_{k-1}}(Y_{s_{k-1}})-\hat m_{T-s_{k-1}}(Y_{s_{k-1}})\big\|_2^2
\right]
\int_{s_{k-1}}^{s_k}\frac{ds}{(T-s)^2}.
\label{eq:approx-sep}
\end{align}
But
\[
\int_{s_{k-1}}^{s_k}\frac{ds}{(T-s)^2}
=\left[\frac{1}{T-s}\right]_{s_{k-1}}^{s_k}
=\gamma_k-\gamma_{k-1}.
\]
Also $Y_{s_{k-1}}=X_{T-s_{k-1}}=X_{1/\gamma_{k-1}}$, so with $t_{k-1}:=T-s_{k-1}=1/\gamma_{k-1}$,
\[
\E\!\left[
\big\|m_{T-s_{k-1}}(Y_{s_{k-1}})-\hat m_{T-s_{k-1}}(Y_{s_{k-1}})\big\|_2^2
\right]
=
\E\!\left[\big\|m_{t_{k-1}}(X_{t_{k-1}})-\hat m_{t_{k-1}}(X_{t_{k-1}})\big\|_2^2\right].
\]
Denote $m:=m_{t_{k-1}}(X_{t_{k-1}})=\E[Z\mid X_{t_{k-1}}]$, $\hat m:=\hat m_{t_{k-1}}(X_{t_{k-1}})$. Then
\begin{align*}
\E\|Z-\hat m\|^2
&=\E\|Z-m+m-\hat m\|^2 \\
&=\E\|Z-m\|^2+\E\|m-\hat m\|^2
+2\,\E\!\big[(Z-m)^\top(m-\hat m)\big].
\end{align*}
The cross term is zero by conditional orthogonality.
Therefore,
\[
\E\|m-\hat m\|^2=\E\|Z-\hat m\|^2-\E\|Z-m\|^2
=L_{k-1}-\mathrm{mmse}(\gamma_{k-1}).
\]
Plugging into \eqref{eq:approx-sep} yields
\begin{equation}
\label{eq:approx-expand}
\mathcal E_{\mathrm{apx}}
=\sum_{k=1}^K(\gamma_k-\gamma_{k-1})\Big(\mathcal L_{x_0}(\gamma_{k-1})-\mathrm{mmse}(\gamma_{k-1})\Big).
\end{equation}

Finally, add \eqref{eq:disc-expand} and \eqref{eq:approx-expand}: the terms
$\sum_{k=1}^K(\gamma_k-\gamma_{k-1})\,\mathrm{mmse}(\gamma_{k-1})$ cancel, giving \eqref{eq:disc+approx}.
\end{proof}


\section{Schedule Optimization Algorithms}
\label{app:schedule-optimization}

This appendix describes the algorithms used to compute the discretization schedule from a finite candidate set of
SNR values. Let $\{\gamma_i\}_{i=0}^{n-1}$ be candidate SNRs (sorted increasing), and let $L(i)$ denote the
estimated diagnostic risk at $\gamma_i$ (e.g., an $x_0$-prediction risk or a known rescaling of the
$\varepsilon$-loss). We fix endpoints $i_0=0$ and $i_K=n-1$.
Define the regularized SNR axis
\[
\eta(\gamma) := \frac{\gamma}{1+\lambda^2\gamma},\qquad
\eta_i := \eta(\gamma_i),\qquad
\ell_i := \log\gamma_i .
\]
For a schedule $i_0<i_1<\cdots<i_K$, define log-SNR step sizes
\[
h_k := \log\frac{\gamma_{i_k}}{\gamma_{i_{k-1}}} = \ell_{i_k}-\ell_{i_{k-1}} .
\]
We minimize the surrogate objective
\begin{equation}
\label{eq:app-objective}
\sum_{k=1}^{K}(\eta_{i_k}-\eta_{i_{k-1}})\,L(i_{k-1})
\;+\;
\alpha \sum_{k=2}^{K}(h_k-h_{k-1})^2 ,
\end{equation}
where $\alpha=0$ yields a first-order objective (e.g., DDIM), and $\alpha>0$ adds the smoothness penalty used for
second-order samplers (e.g., SDE-DPM++(2M)).

\subsection{Exact Dynamic Programming for $\alpha=0$}
When $\alpha=0$, the objective in \eqref{eq:app-objective} becomes first-order (the cost of a step depends only on
the current grid point) and can be solved exactly by a shortest-path dynamic program on a Directed Acyclic Graph (DAG).
Algorithm~\ref{alg:dp-alpha0} gives an $O(Kn^2)$ procedure that selects $K{+}1$ increasing indices
$0=i_0<i_1<\cdots<i_K=n-1$ by minimizing the transition costs
$(\eta_{i_k}-\eta_{i_{k-1}})\,L(i_{k-1})$ with fixed endpoints.

\subsection{Heuristic Beam-and-Window Dynamical Programming (DP) for $\alpha>0$}
For $\alpha>0$, the smoothness penalty couples consecutive log-steps:
$(h_k-h_{k-1})^2$ depends on the triple $(i_{k-2},i_{k-1},i_k)$.
An exact second-order DP over index pairs is possible but naively costs $O(Kn^3)$ due to the additional
minimization over the predecessor at each transition.
In practice, we use a fast approximate shortest-path routine that combines beam pruning with localized candidate
expansion on the log-SNR axis.
Algorithm~\ref{alg:dp-heuristic-alphapos} summarizes the resulting beam-and-window DP, which (i) maintains a bounded
number of partial paths per endpoint (beam width $B$), and (ii) expands only candidate next indices within a window
around a predicted next log-SNR location, optionally augmented with a small set of global candidates.

\paragraph{Prediction rule.}
Given the last two indices $(a,b)$, we predict the next log-SNR via constant log-ratio continuation:
\begin{equation}
\ell_{\mathrm{pred}} \;:=\; \ell_b + (\ell_b-\ell_a) \;=\; 2\ell_b-\ell_a .
\end{equation}
Algorithm~\ref{alg:dp-heuristic-alphapos} then expands a window of indices around the insertion position of
$\ell_{\mathrm{pred}}$ (with radius $W$), which targets schedules with approximately stable log-SNR ratios while
keeping computation inexpensive.

\section{Experiments }\label{app:toy}

This appendix collects supplementary material for the experimental section.
We first record a simple identity that connects the diagnostic used in our schedule objective to the standard
training loss in DDPM parameterizations. We then provide additional details for the toy GMM and ImageNet
experiments reported in the main text.

\begin{remark}
In the DDPM forward process
\begin{equation*}
X_t=\sqrt{\bar\alpha_t}\,X_0+\sqrt{1-\bar\alpha_t}\,\varepsilon,\qquad \varepsilon\sim\mathcal N(0,I),
\end{equation*}
an $\varepsilon$-predictor $\hat\varepsilon_t(X_t)$ induces the usual $x_0$-predictor
\[
\hat X_0(X_t)=\frac{X_t-\sqrt{1-\bar\alpha_t}\,\hat\varepsilon_t(X_t)}{\sqrt{\bar\alpha_t}}.
\]
Hence, with $\gamma_t:=\frac{\bar\alpha_t}{1-\bar\alpha_t}$ (SNR),
\begingroup
\small
\begin{equation*}
\label{eq:eps_x0_identity}
\|X_0-\hat X_0(X_t)\|_2^2
=\frac{1-\bar\alpha_t}{\bar\alpha_t}\,\|\varepsilon-\hat\varepsilon_t(X_t)\|_2^2
=\frac{1}{\gamma_t}\,\|\varepsilon-\hat\varepsilon_t(X_t)\|_2^2.
\end{equation*}
\endgroup
Therefore, the $x_0$-prediction risks appearing in our schedule objective can be obtained directly from the
standard $\varepsilon$-training losses via a known SNR factor, with little extra post-training computation.
\end{remark}

\paragraph{Toy GMM experiments.}
We evaluate schedules on two $2$D $8$-component isotropic GMM priors with $\sigma_0=0.25$.
\texttt{circle8} places means uniformly on a radius-$4$ circle with fixed non-uniform weights, and
\texttt{grid8} places means on a $2\times4$ grid with fixed non-uniform weights.
For each NFE $K\in\{5,7,10\}$, we compare the optimized schedule against linear-time and EDM ($\rho=7$),
using DDIM ($\eta=1$) and SDE-DPM-Solver++(2M) samplers. We generate $20{,}000$ samples per setting and report the
negative mean log-likelihood under the true GMM (lower is better).

Tables~\ref{tab:circle8_nll}--\ref{tab:grid8_nll} show that LAS consistently achieves the best negative mean
log-likelihood across both priors and both samplers, with the largest improvements at low NFE (notably $K=5$
and $K=7$). Linear-time is second-best, while EDM performs worst in these toy settings.

\paragraph{ImageNet $256\times256$ (latent diffusion) details.}
We evaluate LAS on ImageNet $256\times256$ using a latent diffusion model with classifier guidance scale $2$.
We tune the SNR-axis regularization parameter by grid search over
$\lambda\in\{0.5,1.0,1.5,2.0\}$ using a pilot budget of $1{,}000$ generated samples and fix the best value
$\lambda=1.5$ for all subsequent ImageNet experiments.
For the DPM sampling, we additionally tune the exponent by grid search over
$\alpha \in\{4,8,10,12,15\}$ (with all other settings fixed) and fix the best value $\alpha=12$ for all
DPM-Solver experiments.

For DDIM ($\eta=1$), the LAS schedules (timesteps, noisy $\rightarrow$ clean) are:
\[
K=10:\ [999,746,607,527,462,402,342,280,208,126,0],
\]
\[
K=20:\ [999,808,704,635,583,543,509,479,450,422,392,362,332,300,268,234,198,158,112,60,0].
\]
For SDE-DPM-Sovler++(2M) and DPM-Solver++(2M), the LAS schedules are:
\[
K=10:\ [999,848,690,553,460,380,302,210,98,14,0],
\]
\[
K=20:\ [999,905,804,708,627,569,523,485,448,414,380,342,306,268,226,180,126,72,26,6,0].
\]

\begin{table}[t]
\centering
\caption{Negative mean log-likelihood (lower is better) on circle8.}
\label{tab:circle8_nll}
\begin{tabular}{ll|lccc}
\toprule
Sampler type & Sampler & Schedule & NFE=5 & NFE=7 & NFE=10\\
\midrule
\multirow{6}{*}{Stochastic}
 & \multirow{3}{*}{DDIM ($\eta=1$)}
 & LAS & 1.937 & 1.639 & 1.609\\
 &  & Linear-time & 4.063 & 2.408 & 1.810\\
 &  & EDM  & 6.748 & 4.013 & 2.546\\
\cmidrule(lr){2-6}
 & \multirow{3}{*}{SDE-DPM++(2M)}
 & LAS& 1.721 & 1.500 & 1.574\\
 &  & Linear-time & 3.314 & 2.085 & 1.662\\
 &  & EDM & 5.585 & 7.272 & 3.897\\
\bottomrule
\end{tabular}
\end{table}

\begin{table}[t]
\centering
\caption{Negative mean log-likelihood (lower is better) on grid8.}
\label{tab:grid8_nll}
\begin{tabular}{ll|lccc}
\toprule
Sampler type & Sampler & Schedule & NFE=5 & NFE=7 & NFE=10\\
\midrule
\multirow{6}{*}{Stochastic}
 & \multirow{3}{*}{DDIM ($\eta=1$)}
 & LAS  & 2.077 & 1.758 & 1.650\\
 &  & Linear-time & 4.310 & 2.516 & 1.849\\
 &  & EDM & 5.832 & 4.047 & 2.604\\
\cmidrule(lr){2-6}
 & \multirow{3}{*}{SDE-DPM++(2M) }
 & LAS & 1.875 & 1.569 & 1.586\\
 &  & Linear-time & 3.553 & 2.188 & 1.689\\
 &  & EDM & 5.319 & 6.713 & 3.820\\
\bottomrule
\end{tabular}
\end{table}

\begin{algorithm}[tb]
\caption{Exact schedule optimization for $\alpha=0$ (first-order DP)}
\label{alg:dp-alpha0}
\begin{algorithmic}[1]
\REQUIRE Candidates $\{\gamma_i\}_{i=0}^{n-1}$ (increasing), risks $\{L(i)\}$, steps $K$, parameter $\lambda>0$.
\ENSURE Indices $0=i_0<i_1<\cdots<i_K=n-1$ minimizing \eqref{eq:app-objective} with $\alpha=0$.

\STATE Compute $\eta_i \leftarrow \gamma_i/(1+\lambda^2\gamma_i)$ for all $i$.
\STATE $\texttt{end} \leftarrow n-1$, $\texttt{INF}$ large.
\STATE Allocate $\texttt{dp}[1{:}K,0{:}\texttt{end}] \leftarrow \texttt{INF}$ and $\texttt{par}[1{:}K,0{:}\texttt{end}] \leftarrow -1$.
\COMMENT{$\texttt{dp}[k,j]$: best cost to reach $j$ in exactly $k$ transitions from $0$}

\FOR{$j = 1$ \textbf{to} $\texttt{end}-(K-1)$}
  \STATE $\texttt{dp}[1,j] \leftarrow (\eta_j-\eta_0)\,L(0)$; \ \ $\texttt{par}[1,j]\leftarrow 0$.
\ENDFOR

\FOR{$k = 2$ \textbf{to} $K-1$}
  \STATE $\texttt{maxJ} \leftarrow \texttt{end}-(K-k)$ 
  \FOR{$j = k$ \textbf{to} $\texttt{maxJ}$}
    \STATE $\texttt{dp}[k,j] \leftarrow \min\limits_{i<j}\Big\{\texttt{dp}[k-1,i] + (\eta_j-\eta_i)\,L(i)\Big\}$.
    \STATE $\texttt{par}[k,j] \leftarrow \arg\min\limits_{i<j}\Big\{\texttt{dp}[k-1,i] + (\eta_j-\eta_i)\,L(i)\Big\}$.
  \ENDFOR
\ENDFOR

\STATE $i_K \leftarrow \texttt{end}$.
\STATE $i_{K-1} \leftarrow \arg\min\limits_{i<i_K}\Big\{\texttt{dp}[K-1,i] + (\eta_{i_K}-\eta_i)\,L(i)\Big\}$.
\STATE Backtrack $i_{K-2},\dots,i_0$ using $\texttt{par}$ and return $(i_0,\dots,i_K)$.
\end{algorithmic}
\end{algorithm}

\begin{algorithm}[tb]
\caption{Heuristic schedule optimization for $\alpha>0$ (beam-pruned windowed DP)}
\label{alg:dp-heuristic-alphapos}
\begin{algorithmic}[1]
\REQUIRE Candidates $\{\gamma_i\}_{i=0}^{n-1}$ (increasing), risks $\{L(i)\}$, steps $K$, $\lambda>0$, $\alpha>0$.
\REQUIRE Beam width $B$, window radius $W$, extra candidates $E$.
\ENSURE Approximate minimizer of \eqref{eq:app-objective}: indices $0=i_0<i_1<\cdots<i_K=n-1$.

\STATE Compute $\eta_i \leftarrow \gamma_i/(1+\lambda^2\gamma_i)$ and $\ell_i \leftarrow \log\gamma_i$ for all $i$.
\STATE $\texttt{end}\leftarrow n-1$.
\STATE A \emph{state} is $(a,b,cost)$ representing the last two indices $(a,b)$ and accumulated cost.
\STATE Let $\mathcal{S}_k(b)$ be a list of up to $B$ states that end at index $b$ after $k$ transitions.
\FOR{$b = 1$ \textbf{to} $\texttt{end}-(K-1)$}
  \STATE $\mathcal{S}_1(b) \leftarrow \{(0,b,\,(\eta_b-\eta_0)L(0))\}$.
\ENDFOR

\FOR{$k = 2$ \textbf{to} $K-1$}
  \STATE $\texttt{maxIdx} \leftarrow \texttt{end}-(K-k)$.
  \STATE Initialize all $\mathcal{S}_k(\cdot)$ to empty.
  \FOR{each endpoint $b$ with $\mathcal{S}_{k-1}(b)\neq\emptyset$}
    \FOR{each $(a,b,cost)\in \mathcal{S}_{k-1}(b)$}
      \STATE $\ell_{\mathrm{pred}} \leftarrow 2\ell_b-\ell_a$.
      \STATE $j \leftarrow \min\{j:\ell_j \ge \ell_{\mathrm{pred}}\}$

      \STATE $\mathcal{C} \leftarrow \{c:\max(b{+}1,j{-}W)\le c\le \min(\texttt{maxIdx},j{+}W)\}$.
      \FOR{each $c\in\mathcal{C}$}
        \STATE $\Delta_{\mathrm{base}} \leftarrow (\eta_c-\eta_b)\,L(b)$.
        \STATE $\Delta_{\mathrm{sm}} \leftarrow \alpha\Big((\ell_c-\ell_b)-(\ell_b-\ell_a)\Big)^2$.
        \STATE $\texttt{newCost} \leftarrow cost + \Delta_{\mathrm{base}} + \Delta_{\mathrm{sm}}$.
        \STATE Insert $(b,c,\texttt{newCost})$ into $\mathcal{S}_k(c)$ with a backpointer to $(a,b)$.
      \ENDFOR
    \ENDFOR
  \ENDFOR
  \STATE \textbf{Beam pruning:} for each $c$, keep only the $B$ states in $\mathcal{S}_k(c)$ with smallest cost.
\ENDFOR

\STATE $i_K \leftarrow \texttt{end}$.
\STATE Among all states $(a,b,cost)\in \mathcal{S}_{K-1}(b)$, select the one minimizing
\[
cost + (\eta_{i_K}-\eta_b)L(b) + \alpha\Big((\ell_{i_K}-\ell_b)-(\ell_b-\ell_a)\Big)^2 .
\]
\STATE Backtrack pointers to recover $(i_0,\dots,i_K)$ and return.
\end{algorithmic}
\end{algorithm}
\end{document}